\documentclass{article}

\usepackage{microtype}
\usepackage{graphicx}
\usepackage{subfigure}
\usepackage{booktabs} 

\usepackage{hyperref}

\usepackage[accepted]{icml2025}

\usepackage{mathtools,amssymb,amsthm,bm,centernot}
\usepackage[Symbol]{upgreek}
\usepackage{float}

\newcommand\Ed{\mathbb{E}}
\newcommand\du{\mathrm{d}}

\newcommand\Rd{\mathbb{R}}

\newcommand\Bc{\mathcal{B}}
\newcommand\Nc{\mathcal{N}}
\newcommand\rbr[1]{\left(#1\right)}
\newcommand\sbr[1]{\left[#1\right]}
\newcommand\cbr[1]{\left\{#1\right\}}

\newcommand\abs[1]{\left|#1\right|}
\newcommand\Var{\operatorname{Var}}

\newcommand\KL{\operatorname{KL}}
\newcommand\ELBO{\operatorname{ELBO}}

\makeatletter
\newcommand{\bigperp}{%
  \mathop{\mathpalette\bigp@rp\relax}%
  \displaylimits
}

\newcommand{\bigp@rp}[2]{%
  \vcenter{
    \m@th\hbox{\scalebox{\ifx#1\displaystyle2.1\else1.5\fi}{$#1\perp$}}
  }%
}
\makeatother

\usepackage{makecell}

\usepackage[capitalize,noabbrev]{cleveref}

\theoremstyle{plain}
\newtheorem{theorem}{Theorem}[section]

\theoremstyle{definition}

\theoremstyle{remark}

\usepackage[textsize=tiny]{todonotes}

\icmltitlerunning{A Revisit of Total Correlation in Disentangled Variational Auto-Encoder with Partial Disentanglement}

\begin{document}

\twocolumn[
\icmltitle{A Revisit of Total Correlation in Disentangled Variational Auto-Encoder with Partial Disentanglement}



\icmlsetsymbol{equal}{*}

\begin{icmlauthorlist}
\icmlauthor{Chengrui Li}{gatech}
\icmlauthor{Yunmiao Wang}{emory}
\icmlauthor{Yule Wang}{gatech}
\icmlauthor{Weihan Li}{gatech}
\icmlauthor{Dieter Jaeger}{emory}
\icmlauthor{Anqi Wu}{gatech}
\end{icmlauthorlist}

\icmlaffiliation{gatech}{School of Computational Science \& Engineering, Georgia Institute of Technology, Atlanta, GA, USA}
\icmlaffiliation{emory}{Department of Biology, Emory University, Atlanta, GA, USA}

\icmlcorrespondingauthor{Anqi Wu}{anqiwu@gatech.edu}

\icmlkeywords{disentangling variational auto-encoder, independent component analysis, neural subspace, neuroscience}

\vskip 0.3in
]



\printAffiliationsAndNotice{}  

\begin{abstract}
A fully disentangled variational auto-encoder (VAE) aims to identify disentangled latent components from observations. However, enforcing full independence between all latent components may be too strict for certain datasets. In some cases, multiple factors may be entangled together in a non-separable manner, or a single independent semantic meaning could be represented by multiple latent components within a higher-dimensional manifold. To address such scenarios with greater flexibility, we develop the Partially Disentangled VAE (PDisVAE), which generalizes the total correlation (TC) term in fully disentangled VAEs to a partial correlation (PC) term. This framework can handle group-wise independence and can naturally reduce to either the standard VAE or the fully disentangled VAE. Validation through three synthetic experiments demonstrates the correctness and practicality of PDisVAE. When applied to real-world datasets, PDisVAE discovers valuable information that is difficult to find using fully disentangled VAEs, implying its versatility and effectiveness.
\end{abstract}

\section{Introduction}
Disentangling independent latent components from observations is a desirable goal in representational learning \citep{bengio2013representation, alemi2016deep, schmidhuber1992learning, achille2017emergence}, with numerous applications in fields such as computer vision and image processing \citep{lake2017building}, signal analysis \citep{hyvarinen2000independent, hyvarinen2017nonlinear}, and neuroscience \citep{zhou2020learning, yang2021causalvae, wang2024exploring, calhoun2009review}. To disentangle latent components in an unsupervised manner, most models employ techniques that combine optimizing a variational auto-encoder (VAE) \citep{kingma2013auto} with an additional penalty term known as total correlation (mutual information) \citep{kraskov2004estimating}, classified as fully disentangled VAEs \citep{higgins2017beta, kim2018disentangling, chen2018isolating}.

However, enforcing full independence among all latent components can be an overly strong assumption for certain datasets. For instance, consider the location coordinates $(x, y)$ of a set of points in a 2D plane. If the points are uniformly distributed within a square $[-1, 1] \times [-1, 1]$, the location distribution can be expressed as $p(x,y) = p(x)p(y)$, indicating that $x$ and $y$ are independent components. However, if the points are distributed in an irregular shape, such as a butterfly, the $(x, y)$ coordinates become entangled, resulting in $p(x, y) \neq p(x)p(y)$. In this case, the location information cannot be decomposed into two independent components but must be jointly represented by $(x,y)$ together. If the points also have attributes independent of their location, such as RGB color represented by a 3D vector, we then encounter the \textbf{group-wise independence}, where a rank-2 entangled group (location) is independent of a rank-3 entangled group (color).

\begin{table}[!ht]
    \centering
    \caption{Different unsupervised disentangling methods. Other related methods are discussed in Appendix.~\ref{appendix:related_works}.}
    \begin{tabular}{lcc}
        \toprule
        Disentanglement type & full & partial \\
        \midrule
        By prior (not flexible) & ICA & ISA-VAE \\
        By penalty (flexible) & \{Factor, $\beta$\}-VAE & \textbf{PDisVAE} \\
        \bottomrule
    \end{tabular}
    \label{tab:different_methods}
\end{table}

To deal with such group-wise independence, we develop the \textbf{partially disentangled VAE (PDisVAE)}.\\
$\bullet$ First, it achieves group-wise independence by generalizing the total correlation (TC) penalty term in the loss function of fully disentangled VAEs to partial correlation (PC), instead of a rigidly defined group-wise independent prior used in ISA-VAE \citep{stuhmer2020independent}. PC explicitly penalizes group-wise independence while permitting within-group entanglement flexibly. This unified formulation of PC encompasses both the standard VAE and fully disentangled VAEs. Tab.~\ref{tab:different_methods} compare these differences. Other related works are summarized in Appendix.~\ref{appendix:related_works}.\\
$\bullet$ Second, we revisit the batch approximation method used for computing PC and TC. The existing batch approximation method proposed by \citet{chen2018isolating} for computing TC in fully disentangled VAEs exhibits a high variance in the estimator. Since accurate batch approximation is critical for the success of the method, we derive the optimal importance sampling (IS) batch approximation formula and provide a theoretical proof of its optimality.

\section{Backgrounds: fully disentangled VAEs}\label{sec:FDisVAE}
\subsection{By total correlation (TC)}
Given a dataset of observations $\cbr{\bm{x}^{(n)}}_{n=1}^N$ consisting of $N$ samples, fully disentangled VAEs aim to identify $K$ statistically independent (disentangled) latent components, $z_1 \perp \dots \perp z_K$, within the latent variable $\bm{z} \in \mathbb{R}^K$ that generate the observation $\bm{x}\in\Rd^D$. To achieve full disentanglement, fully disentangled VAEs optimize:
\begin{equation}\label{eq:FDisVAE}
     \mathcal{L} = \frac{1}{N} \sum_{n=1}^N \ELBO\rbr{\bm x^{(n)}} - \beta \cdot \KL\rbr{q(\bm z) \middle\| \prod_{k=1}^K q(z_k)},
\end{equation}
where $\ELBO(\bm x^{(n)}) = \Ed_{q(\bm z|\bm x^{(n)})}\sbr{\ln p\rbr{\bm x^{(n)}|\bm z}} - \KL\rbr{q\rbr{\bm z\middle|\bm x^{(n)}}\middle\|p(\bm z)}$ \citep{blei2017variational} is the standard VAE loss. In these formulae, $p(\bm x|\bm z;\theta) = \Nc\rbr{\bm x;\bm\mu, \bm\sigma^2},\ \bm\mu, \bm\sigma^2 = \operatorname{decoder}(\bm z;\theta)$, where $\operatorname{decoder}:\Rd^K \to \Rd^D$ is parameterized by $\theta$. $q(\bm z|\bm x;\phi)= \Nc(\bm z;\bm \mu=, \bm\sigma^2),\ \bm\mu, \bm\sigma^2 = \operatorname{encoder}(\bm x;\theta)$ is the variational distribution, in which the $\operatorname{encoder}:\Rd^{D} \to \Rd^K$ is parameterized by $\phi$. In Eq.~\eqref{eq:FDisVAE} and the following, we omit $\theta$ in $p$ and $\phi$ in $q$ for simplification. The prior $p(\bm z)$ is often chosen to be a standard normal prior. The second term in Eq.~\eqref{eq:FDisVAE} is the total correlation (TC), where $q(\bm z) = \frac{1}{N}\sum_{n=1}^N q(\bm z,n) = \sum_{n=1}^N q\rbr{\bm z\middle|\bm x^{(n)}} q(n)$ is the aggregated posterior, followed by \citet{makhzani2015adversarial}. Specifically, given a dataset of $N$ equally treated samples, the probability of taking the $n$-th sample is $q(n) = \frac{1}{N}$, so that $\frac{1}{N} \sum_{n=1}^N [\cdot] = \Ed_{q(n)}[\cdot]$. Also let $q(\bm z|n) \coloneqq q\rbr{\bm z\middle|\bm x^{(n)}}$, then $q(\bm z)$ can be viewed as a Gaussian kernel density estimation from $\cbr{\bm z^{(n)}}_{n=1}^N$ in latent space. The goal of this TC term is to achieve $q(\bm z) =\prod_{k=1}^K q(z_k)$, which is the rigorous definition of independence among $z_1,...,z_k$. That is why Eq.~\eqref{eq:FDisVAE} can achieve full disentanglement compared with standard VAE.

Before the development Eq.~\eqref{eq:FDisVAE}, \citet{higgins2017beta} and \citet{burgess2018understanding} initially discovered that penalizing the entire KL divergence in $\ELBO$ can increase the latent disentanglement, resulting their $\beta$-VAE. It was found later by \citet{kim2018disentangling} and \citet{chen2018isolating} and summarized by \citet{dubois2019dvae} that the effective term for enhancing the latent disentanglement is indeed the TC. Consequently, they developed Eq.~\eqref{eq:FDisVAE} with $\beta > 0$, resulting in FactorVAE and $\beta$-TCVAE representing the fully disentangled VAEs.

\subsection{By a non-Gaussian prior (ICA)}
Another approach to achieving full disentanglement is to view the problem as an independent component analysis (ICA). The core idea inspired by ICA is that ``non-Gaussian is independent'' \citep{hyvarinen2000independent,hyvarinen2009independent}. In short, we need to assume $p(\bm z)$ to be non-Gaussian. The $\operatorname{logcosh}$ distribution is one of the most commonly used:
\begin{equation}\label{eq:logcosh}
    p(\bm z) = \prod_{k=1}^K p(z_k) = \prod_{k=1}^K \frac{\uppi\rbr{\operatorname{sech}\frac{\uppi z_k}{2\sqrt{3}}}^2}{4\sqrt{3}}.
\end{equation}

In traditional linear ICA, $\bm x=\bm f(\bm z)$ where $\bm f:\Rd^K \to \Rd^D$ is a full-rank ($D=K$) linear deterministic mapping, and $p(\bm x|\bm z;\bm f) = \delta(\bm x - \bm f(\bm z))$ ($\delta$ is the Dirac delta function), then we can use maximum likelihood estimate (MLE) to learn $\bm f$ via the ``change of variable'' formula,
\begin{equation}
    p(\bm x) = \int p(\bm x|\bm z;\bm f) p(\bm z)\ \du \bm z = \abs{\det \frac{\du \bm f^{-1}}{\du \bm z}}\cdot p(\bm f^{-1}(\bm x)),
\end{equation}

and recover $\bm z = \bm f^{-1}(\bm x)$. However, there are two main drawbacks. First, it cannot be extended to non-invertible non-linear $\bm f(\bm z)$ since the $\abs{\det \frac{\du \bm f^{-1}}{\du \bm z}}$ in the ``change of variable'' formula becomes intractable \citep{khemakhem2020variational, sorrenson2020disentanglement}. Second, $\bm x\in\Rd^D$ is usually in higher dimensional space than $\bm z\in\Rd^K$ ($D > K$) with noises, which are not explicitly modeled by traditional linear ICA. 

To address these issues, we use a VAE with a logcosh prior $p(\bm{z})$ defined in Eq.~\eqref{eq:logcosh}. It is worth mentioning that, to the best of our knowledge, we are the first to recognize the logcosh-priored VAE as the nonlinear ICA problem. However, certain limitations remain. For instance, if the true number of disentangled latent components is two but we instruct the logcosh-priored VAE to find three, it will yield three components with poor disentanglement instead of finding two disentangled components and one non-informative component. We will discuss this limitation in detail in the experiment section. Additionally, the logcosh-priored VAE cannot be extended to a partially disentangled version, since the logcosh prior does not support partial independence.

\section{Partially disentangled VAE (PDisVAE)}\label{sec:PDisVAE}
\subsection{Problem definition}\label{sec:problem_definition}
Although several approaches have been introduced in Sec.~\ref{sec:FDisVAE}, a common issue among them is they are all trying to find ``fully disentangled (independent)'' latent space. However, if the true latent variables are partially disentangled by groups, applying a fully disentangled method is hard to successfully recover the underlying latent structure accurately.

We first formally define partial disentanglement (independence). Still, assume latent $\bm z\in\Rd^K$, but now the latent dimensions are disentangled by $G$ groups, while each group has its internal within-group rank $H$, satisfying $K = G \times H$. For simplicity, we denote the $g$-th group as $\bm z_g = (z_{(g-1)H+1}, \dots, z_{gH})$, so that $\bm z = (\bm z_1, \dots, \bm z_G)$. Then, the \textbf{partially disentangled} latent can be formulated as
\begin{equation}\label{eq:partial_disentanglement}
    \bigperp_{g=1}^G (z_{(g-1)H + 1}, \dots, z_{gH}).
\end{equation}
This equation expresses that within each group, latent components may exhibit dependencies and may not be further disentangled. However, the groups themselves remain independent of each other. We refer to this as \textbf{group-wise independence}. For example, when $K = 6$ and there are $G=3$ groups, the three groups are independent of each other as $(z_1, z_2) \perp (z_3, z_4) \perp (z_5, z_6) \iff p(z_1,\dots,z_6) = p(z_1,z_2)p(z_3,z_4)p(z_5,z_6)$, while dimensions within each group can be highly dependent and might not be further decomposed, i.e., $p(z_1,z_2)\neq p(z_1)p(z_2),\ p(z_3,z_4)\neq p(z_3)p(z_4),\ p(z_5, z_6) \neq p(z_5)p(z_6)$.

To identify partially independent component groups as defined above, one might consider a straightforward approach: using existing methods to impose marginal independence on between-group components. For instance, if we have $(z_1, z_2) \perp z_3$, one might attempt to apply existing algorithms to require $z_1 \perp z_3$ and $z_2 \perp z_3$. However, this is generally NOT correct since the former is a sufficient but not necessary condition ($\implies$) for the latter. A simple counterexample is $p(z_1, z_2, z_3)$ with $p(0, 0, 1) = p(0, 1, 0) = p(1, 0, 0) = p(1, 1, 1) = 0.25$. It can be verified that $(z_1, z_2) \not\perp z_3$, while $z_1 \perp z_3$ and $z_2 \perp z_3$. More detailed explanations are in Appendix~\ref{appendix:marginal_independence}. Therefore, we must explicitly enforce $(z_1, z_2) \perp z_3$.

\subsection{By the $L^p$-nested prior (ISA-VAE)}
To explicitly require group-wise independence, there are still two ways---by a group-wise independent prior or by an extra penalty term to the loss function (see Tab.~\ref{tab:different_methods}). \citet{stuhmer2020independent} extends the ISA-VAE from ICA that utilizes the $L^p$-nested distribution \citep{fernandez1995modeling}
\begin{equation}\label{eq:ISA}
    p(\bm z) = \frac{\psi_0(g(\bm z))}{g(\bm z)^{n-1} \mathcal{S}_g(1)}
\end{equation}
as a group-wise independent prior to achieve the partial disentanglement, where $g$ is an $L^p$-nested function, $\psi_0:\Rd\to\Rd_+$ is the raidal density, and $\mathcal{S}_g(1)$ is the surface area of the $L^0$ nested sphere. More details regarding this approach can be found in the work series of \citet{stuhmer2020independent}, \citet{fernandez1995modeling}, and \citet{sinz2010lp}. However, this approach still needs further investigation. First, synthetic experiments are crucial to validate that a partial disentanglement method can effectively handle group-wise independent ground truth, while it was not conducted in the ISA-VAE paper. Second, similar to fully independent ICA, relying on a predefined prior to achieve group-wise independence might be overly rigid in some cases, as will be illustrated in later sections.

\subsection{By partial correlation (PC)}
To require group-wise independence more flexibly, instead of using a prior, we develop the \textbf{partially disentangled VAE (PDisVAE)} that achieves the group-wise independence by an extra penalty term to the loss. Its target function
\begin{equation}\label{eq:PDisVAE}
     \mathcal{L} = \frac{1}{N} \sum_{n=1}^N \ELBO\rbr{\bm x^{(n)}} - \beta \cdot \KL\rbr{q(\bm z) \middle\| \prod_{g=1}^G q(\bm z_g)}
\end{equation}
replaces the TC term in Eq.~\eqref{eq:FDisVAE} with a partial correlation (PC) term. PC is responsible for disentangling independent groups. When $q(\bm z) = \prod_{g=1}^G q(\bm z_g)$, $\operatorname{PC} = \KL\rbr{q(\bm z) \middle\| \prod_{g=1}^G q(\bm z_g)} = 0$. Otherwise, $\operatorname{PC} > 0$ and is penalized by the hyperparameter $\beta > 0$.

It is worth noting that when $G = 1$, $\operatorname{PC} \equiv 0$ and Eq.~\eqref{eq:PDisVAE} becomes the standard VAE objective function; when $G = K$, PC is just the total correlation (TC) and Eq.~\eqref{eq:PDisVAE} becomes Eq.~\eqref{eq:FDisVAE}, the fully disentangled VAE loss. Compared with ISA-VAE \citep{stuhmer2020independent}, which relies on a predefined group-wise independent prior, utilizing PC to achieve group-wise independence offers greater flexibility by allowing the within-group disentanglement rank to vary, rather than being fixed to a specific rank $H$ in ISA-VAE. This flexibility and effectiveness of our PDisVAE leveraging the PC penalty term will be demonstrated in the next subsection and validated through experiments.

\subsection{The behavior of PDisVAE}
\begin{figure*}[!ht]
    \centering
    \includegraphics[width=\linewidth]{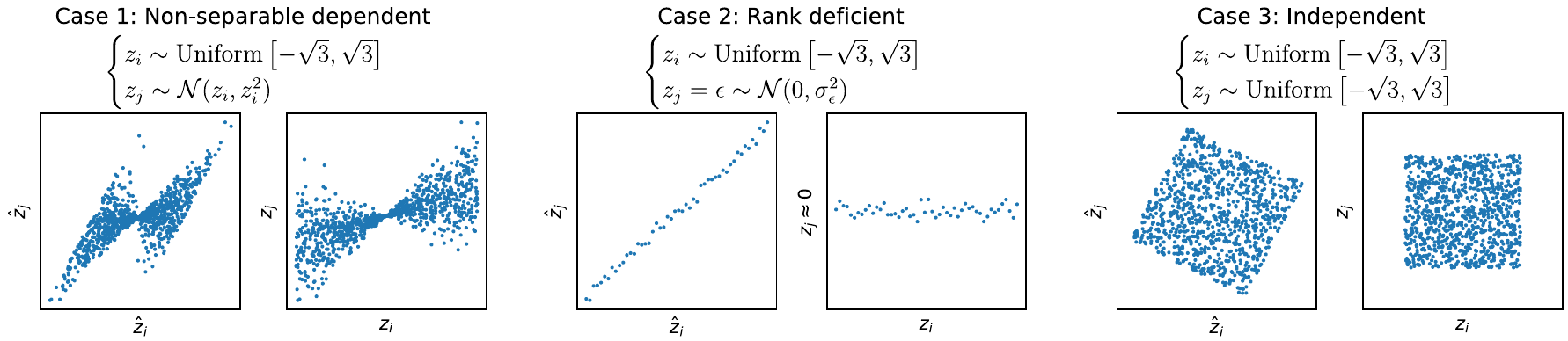}
    \vspace{-0.3in}
    \caption{Visual illustrations for the desired behavior of the PDisVAE. In each case, the left plot is the estimated latent $(\hat z_i, \hat z_j)$ and the right plot is the true latent $(z_i, z_j)$.}
    \label{fig:illustration}
\end{figure*}
In the previous subsection, we introduced PDisVAE but did not discuss what to expect within the groups discovered by PDisVAE. Here, we will outline three potential relationships that the latent components within a group could exhibit. To illustrate, let us consider a discovered latent pair $(\hat z_i, \hat z_j)$; the three cases of interest are illustrated in Fig.~\ref{fig:illustration}.\\
$\bullet$ \textbf{Case 1: Non-separable dependent.} Consider we have the true latent $(z_i, z_j)$ from the equations shown in the right plot of case 1, where both the mean and variance of the Gaussian $z_j$ are dependent on $z_i$. This makes $z_i$ and $z_j$ highly entangled with each other in one group and it is impossible to further separate them independently by any linear transformation. Then, PDisVAE should identify a group $(\hat{z}_i, \hat{z}_j)$ that cannot be further separated independently through any linear transformation. Furthermore, we should be able to align the estimated $(\hat{z}_i, \hat{z}_j)$ with the true $(z_i, z_j)$ via a linear transformation. In this case, the within-group TC cannot become zero under any linear transformation.\\
$\bullet$ \textbf{Case 2: Rank-deficient.} Consider that PDisVAE has identified an estimated group $(\hat{z}_i, \hat{z}_j)$ in the left plot of case 2. Although they are dependent, they exhibit a clear linear relationship, which means they can be reduced to a single effective component, $z_i$, while $z_j$ serves as a dummy latent component. For example, if we have three latent components such that $(z_1, z_2) \perp z_3$, and we apply PDisVAE with $K = 4 = G \times H = 2 \times 2$, we would expect to find a dummy component $z_4 \approx 0$ in the second group, resulting in $(z_1, z_2) \perp (z_3, z_4 \approx 0)$. To verify the presence of a dummy latent, one could apply principal component analysis (PCA) to the group and identify a significantly small principal component, or conduct a normality test to detect Gaussian noise. Note that ISA-VAE is too rigid to allow rank deficiency within a group, since the dummy Gaussian noise variable conflicts with the predefined prior in Eq.~\eqref{eq:ISA}.\\
$\bullet$ \textbf{Case 3: Independent.} In this example, $\hat{z}_i$ and $\hat{z}_j$ are irreducibly dependent on each other. However, it is possible to further separate them into independent components via a linear transformation, resulting in the right plot that $z_i$ and $z_j$ become uniform distributions independent of each other. Consequently, $\hat{z}_i$ and $\hat{z}_j$ identified by PDisVAE should be allocated to two different groups rather than the same group. In this case, the within-group TC can be reduced to zero after a particular linear transformation. This indicates that as long as PDisVAE accurately identifies enough independent groups, the latent components within each group should not be independent of one another.

\subsection{Batch approximation}
During training, strictly computing the aggregated marginal/group posterior of the form $q(z) = \sum_{n=1}^N q(z|n) q(n) = \frac{1}{N}\sum_{n=1}^N q(z|n)$ might be unfeasible, since we only have a batch of size $M$, denoted as $\Bc_M \coloneqq \cbr{n_1, n_2, \dots, n_M}$ without replacement. Although \citet{chen2018isolating} proposed minibatch weighted sampling (MWS) and minibatch stratified sampling (MSS), we argue that our \textbf{importance sampling (IS)} method, derived below and compared in Tab.~\ref{tab:batch_approximation}, is more effective.

Intuitively, when we only have a batch $\Bc_M \subsetneqq \cbr{1,\dots,N}$ and a sampled $z\sim q(z|n_*)$, where $n_*$ is a specific example point in $\Bc_M$, $q(z|n_*)$ is more likely to be greater than $q(z|n\neq n_*)$ since $z$ is sampled from $q(z|n_*)$. Therefore, we want the remaining $M-1$ points in $\Bc_M \backslash \cbr{n_*}$ to represents the entire dataset excluding $n_*$, i.e., $\cbr{1,2,\dots, N}\backslash \cbr{n_*}$. Hence, an approximation of $q(z)$ at $z\sim q(z|n_*)$ could be
\begin{equation}\label{eq:IS}
    \hat q(z) = \frac{1}{N} q(z|n_*) + \sum_{n\in(\Bc_M\backslash\cbr{n_*})} \frac{N-1}{M-1}\frac{1}{N} q(z|n).
\end{equation}
Since each $q(z)$ is approximated using data points within a batch, it might be beneficial to shuffle the dataset every epoch to change the batch samples. Appendix.~\ref{appendix:batch_approximation} includes the complete derivation of this approximation, explaining why it is called IS approximation and proving its optimality, and an empirical evaluation of the three estimators. Notably, IS is more stable than MSS, since $\Var[\mathrm{IS}]<\Var[\mathrm{MSS}]$.

\begin{table}[htbp]
    \centering
    \caption{Comparison of three batch approximation approaches. See Appendix.~\ref{appendix:batch_approximation} for more details.}
    \label{tab:batch_approximation}
    \begin{tabular}{lcc}
        \toprule
         & mean & variance \\
        \midrule
        MWS & biased & \\
        MSS & unbiased & $\Var[\mathrm{MSS}] = \Var[\mathrm{IS}] + \frac{M-2}{M(M-1)}$ \\
        \textbf{IS} & unbiased & $\Var[\mathrm{IS}] = \frac{(N-M)^2}{M^2(M-1)}$ \\
        \bottomrule
    \end{tabular}
\end{table}

\begin{figure*}[!ht]
    \centering
    \includegraphics[width=\linewidth]{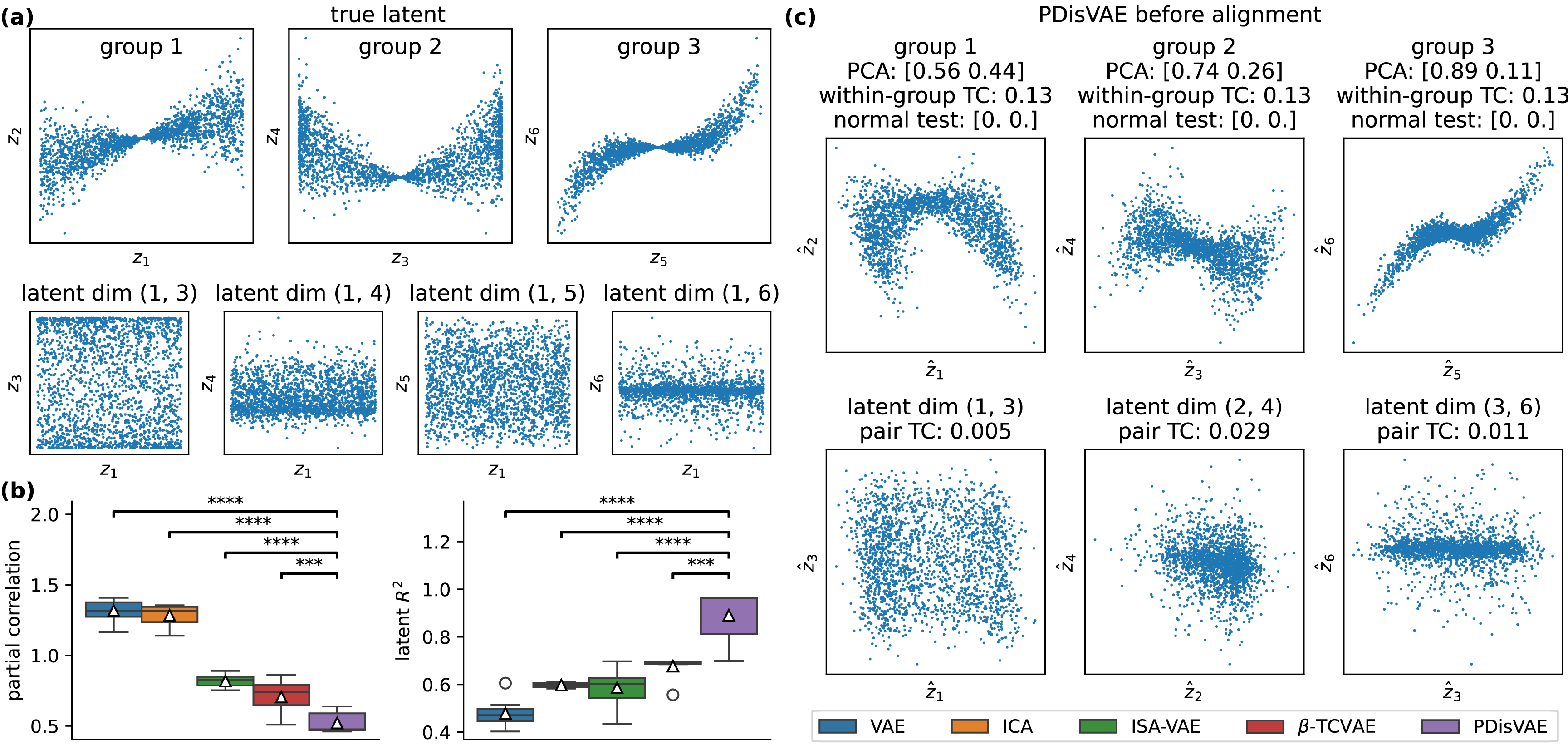}
    \vspace{-0.3in}
    \caption{\textbf{(a)}: The true latent $\bm z \in \Rd^6$ where three groups are $(z_1,z_2)\perp(z_3,z_4)\perp(z_5,z_6)$, but within-groups are highly entangled (top row). Latent components in different groups are marginally independent (bottom row). \textbf{(b)}: The PC of the estimated latent and the latent $R^2$ after alignment to the true latent in (a), with pair-wise $t$-test showing the significance level (***: $p\leqslant 0.001$, ****: $p\leqslant 0.0001$). \textbf{(c)}: The estimated latent of PDisVAE before aligning to the true latent in (a). In each group, PCA shows the explained variance ratio in the group. Within-group TC shows the minimum TC under all possible linear transformations. The normal test shows the $p$-values of the null hypothesis that a marginal distribution is a normal distribution. If $p>0.05$ for example, we may accept the null hypothesis that there exists a Gaussian noise dummy latent component. The pair TC is directly measured from the components in different groups.}
    \label{fig:synthetic_weak_box}
\end{figure*}

\section{Experiments}\label{sec:experiments}
\paragraph{Methods for comparison.} For evaluating the developed PDisVAE, we compare the following methods:\\
$\bullet$ Standard \textbf{VAE} \citep{kingma2013auto}: Theoretically, standard VAE does not have disentaglement ability.\\
$\bullet$ \textbf{ICA}: This is the logcosh-priored VAE for doing non-linear generative ICA inspired by \citet{hyvarinen2000independent}.\\
$\bullet$ \textbf{ISA-VAE} \citep{stuhmer2020independent}: This is the VAE that using the $L^p$-nested prior to achieve group-wise independence.
$\bullet$ \textbf{$\beta$-TCVAE} \citep{chen2018isolating}: This method penalizes an extra TC term to achieve full disentanglement. It is theoretically equivalent to FactorVAE \citep{kim2018disentangling}.\\
$\bullet$ \textbf{PDisVAE}: Our method penalizes the PC term to achieve partial disentanglement, providing a flexible approach to group-wise independent latent. It reduces to the standard VAE when the number of groups $G=1$; and reduces to the fully disentangled VAE when $G=K$ (i.e., the number of groups equals the latent dimensionality). Additionally, it inherently supports within-group rank deficiency.

We will first rigorously validate PDisVAE on two synthetic datasets, then apply it to pdsprites, face images (CelebA), and neural data.


\subsection{Synthetic validation: group-wise independent}
\paragraph{Dataset.} To validate that only PDisVAE is capable of dealing with group-wise independent datasets, we create a dataset consisting of $N=2000$ points in $K=6$ latent space $\bm z^{(n)} \in \Rd^6$, where three groups are independent of each other $(z_1,z_2)\perp(z_3,z_4)\perp(z_5,z_6)$, but components within each group are highly entangled (Fig.~\ref{fig:synthetic_weak_box}(a)) The observations $\bm x$ are linearly mapped from the latents $\bm z$ to a $D = 20$ dimensional space $\bm x^{(n)} \in \Rd^{20}$, and then Gaussian noise $\epsilon_d^{(n)} \overset{i.i.d.}{\sim} \Nc\rbr{0,0.5^2}$ is added.

\paragraph{Experimental setup.} For each method, we use Adam \citep{kingma2014adam} to train a linear encoder and a linear decoder (since the true generative process is linear) for 5,000 epochs. The learning rate is $5\times 10^{-4}$ and the batch size is 128. For $\beta$-TCVAE and PDisVAE, the TC/PC penalty is set as $\beta = 4$. This is supported by \citet{dubois2019dvae}, the $\beta$ selection in $\beta$-TCVAE \citep{chen2018isolating}, and our cross-validation result (Fig.~\ref{fig:synthetic_weak_ablation}) in the ablation study. Each method is run $10$ times with different random seeds.

\paragraph{Results.} The PC box plot in Fig.~\ref{fig:synthetic_weak_box}(b) shows that PDisVAE achieves the lowest PC, implying that PDisVAE disentangles latent in groups the best. Since this is the synthetic dataset and a model match experiment, we can align the estimated latent groups to their corresponding true latent groups to further validate the correctness of the latent estimation. The reconstruction $R^2$ of all methods is approximately $0.97$, indicating that all methods can reconstruct the observation perfectly. However, their learned latent representations are different. The latent $R^2$ in Fig.~\ref{fig:synthetic_weak_box}(b) shows that PDisVAE recovers the latent more accurately than others. Among the alternatives, $\beta$-TCVAE is better than ISA-VAE, ICA, and VAE. It is worth noting that although ISA-VAE is designed to find group-wise independent latent, its performance is not ideal when facing data generated from group-wise independent ground truth latent in practice. Fig.~\ref{fig:synthetic_weak_latent} also visually shows that after aligning with the true latent, PDisVAE recovers the latent the best. 

\begin{figure}[!ht]
    \centering
    \includegraphics[width=\linewidth]{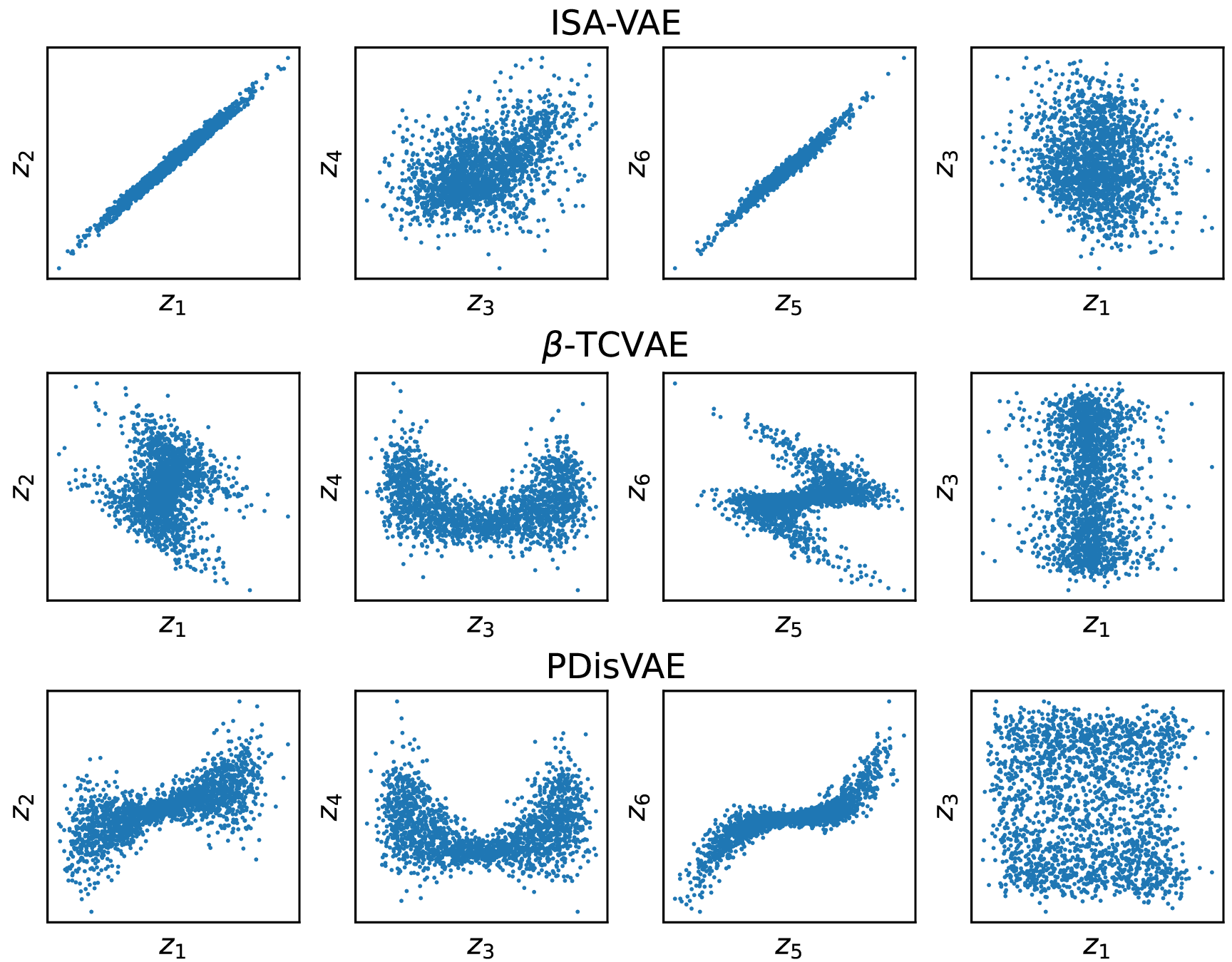}
    \vspace{-0.3in}
    \caption{Estimated latent after aligning to the true latent (Fig.\ref{fig:synthetic_weak_box}(a)) for various methods. Left three columns: the three independent groups; right one column: a between-group component pair. VAE and ICA results are in Fig.~\ref{fig:synthetic_weak_latent_all} in Appendix.~\ref{appendix:supplimentary_results}.}
    \label{fig:synthetic_weak_latent}
\end{figure}

An immediate question that arises is, how to check within-group latent estimated by PDisVAE is truly highly entangled and cannot be further decomposed, especially when there is no true latent. Essentially we hope to find case 1 within a group, rather than case 2 or case 3 illustrated in Fig.~\ref{fig:illustration}. The minimum within-group TC shown in Fig.~\ref{fig:synthetic_weak_box}(c) are all greater than $0$, which means we indeed find highly entangled groups that cannot be further decomposed. Compared to the minimum within-group TC, the close-to-zero pair TC between groups also indicates that components between groups are independent.

\begin{figure*}[!ht]
    \centering
    \includegraphics[width=\linewidth]{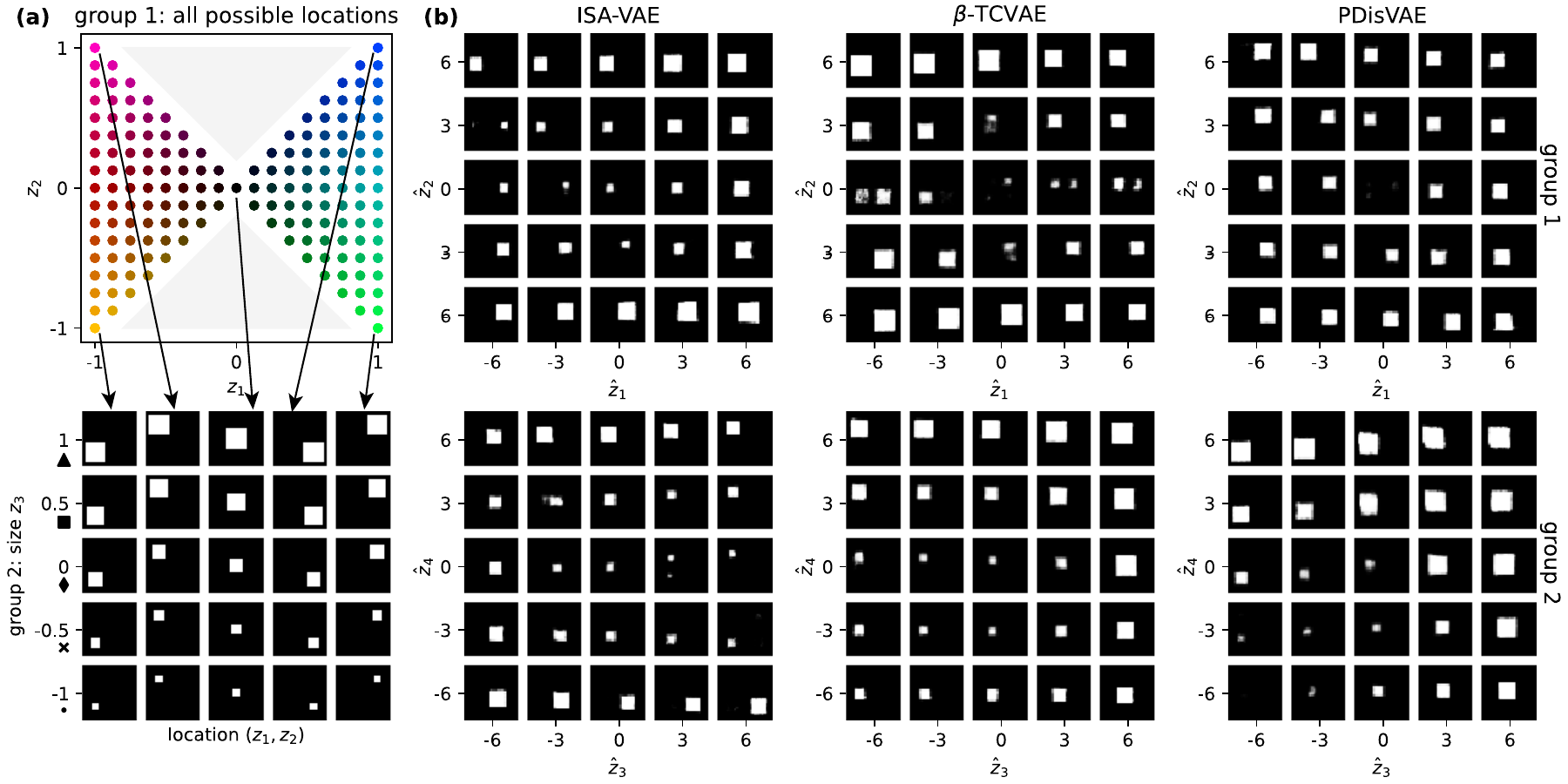}
    \vspace{-0.3in}
    \caption{\textbf{(a)}: Latent and observation generating process. Locations $(z_1, z_2)$ are entangled, and uniformly distributed in a restricted region. Color represents the location information, with the upper and lower gray triangular areas being empty. The size $z_3$ is evenly distributed across five scales, represented by different markers, and is independent of the location. \textbf{(b)}: The reconstructed images by varying one of the latent groups ($(\hat z_1, \hat z_2)$ or $(\hat z_3, \hat z_4)$) found by $\beta$-TCVAE and PDisVAE.}
    \label{fig:pdsprites}
\end{figure*}

\begin{figure*}[!ht]
    \centering
    \includegraphics[width=\linewidth]{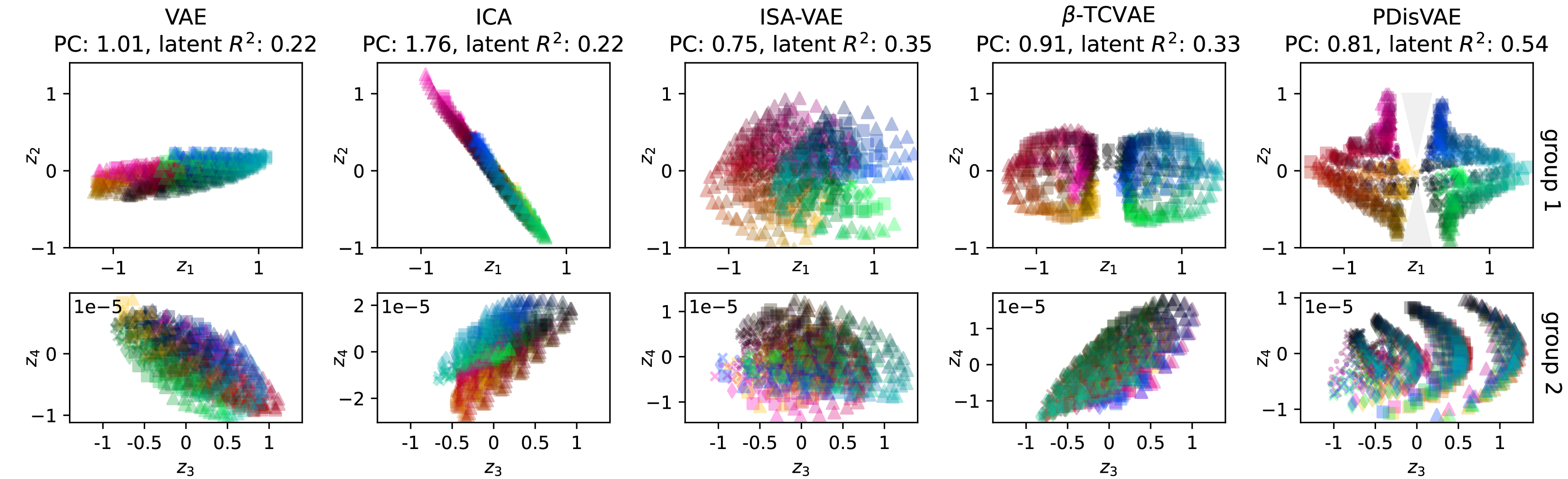}
    \vspace{-0.3in}
    \caption{The latent plot after alignment for the group 1 $(z_1, z_2)$ and group 2 $(z_3, z_4 \approx 0)$ from different methods, and their corresponding PC and latent $R^2$. The color representation for location is the same as the color representation in Fig.~\ref{fig:pdsprites}(a), and the marker of the point in the latent plots represents the size of the square in the observation images.}
    \label{fig:pdsprites_latent_within_group}
\end{figure*}

\paragraph{Ablation.} To analyze the choice of the penalty coefficient $\beta$ of PC term in Eq.~\eqref{eq:PDisVAE}, we vary $\beta$ in PDisVAE from $0.1$ to $100$ and plot the cross-validation results in Fig.~\ref{fig:synthetic_weak_ablation}. The PC and latent $R^2$ plots indicate that $\beta > 1$ is necessary for an accurate recovery and effective minimization of the PC. However, excessively large $\beta$ might negatively impact reconstruction, as shown in the reconstruction $R^2$ plot. Hence, we recommend $\beta\in(2, 10)$, which supports our choice of $\beta = 4$ in our experiments.

\begin{figure}[H]
    \centering
    \includegraphics[width=\linewidth]{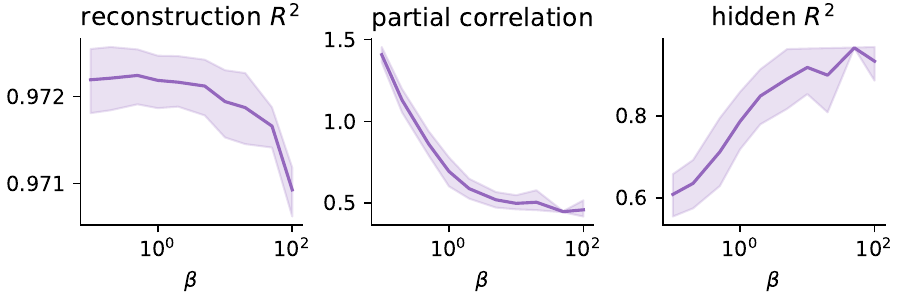}
    \vspace{-0.3in}
    \caption{
    Three metrics w.r.t. the PC coefficient $\beta$ in PDisVAE.}
    \label{fig:synthetic_weak_ablation}
\end{figure}

\paragraph{Flexibly reduce to the fully independent case.} To validate that PDisVAE can flexibly get the same results as from a fully disentangled VAE when the latent is fully independent, we create a dataset that is generated from fully independent latent (Fig.~\ref{fig:synthetic_strong_box}(a) and Fig.~\ref{fig:synthetic_strong_latent}) and apply different methods to it. The PC box plot and latent $R^2$ plot in Fig.~\ref{fig:synthetic_strong_box}(b) show that both $\beta$-TCVAE and PDisVAE achieve the lowest partial correlation and the highest latent $R^2$ on this fully disentangled dataset, which implies that PDisVAE automatically reduces to fully independent result if the group rank is deficient, as illustrated in case 2 in Fig.~\ref{fig:illustration}. In general, the actual group rank can be detected by PDisVAE and if the true group rank is less than the specified group dimensionality, dummy estimated latents will complemented in the corresponding group. More details are in Appendix.~\ref{Appendix:fully_independent}.

\begin{figure*}[!ht]
    \centering
    \includegraphics[width=\linewidth]{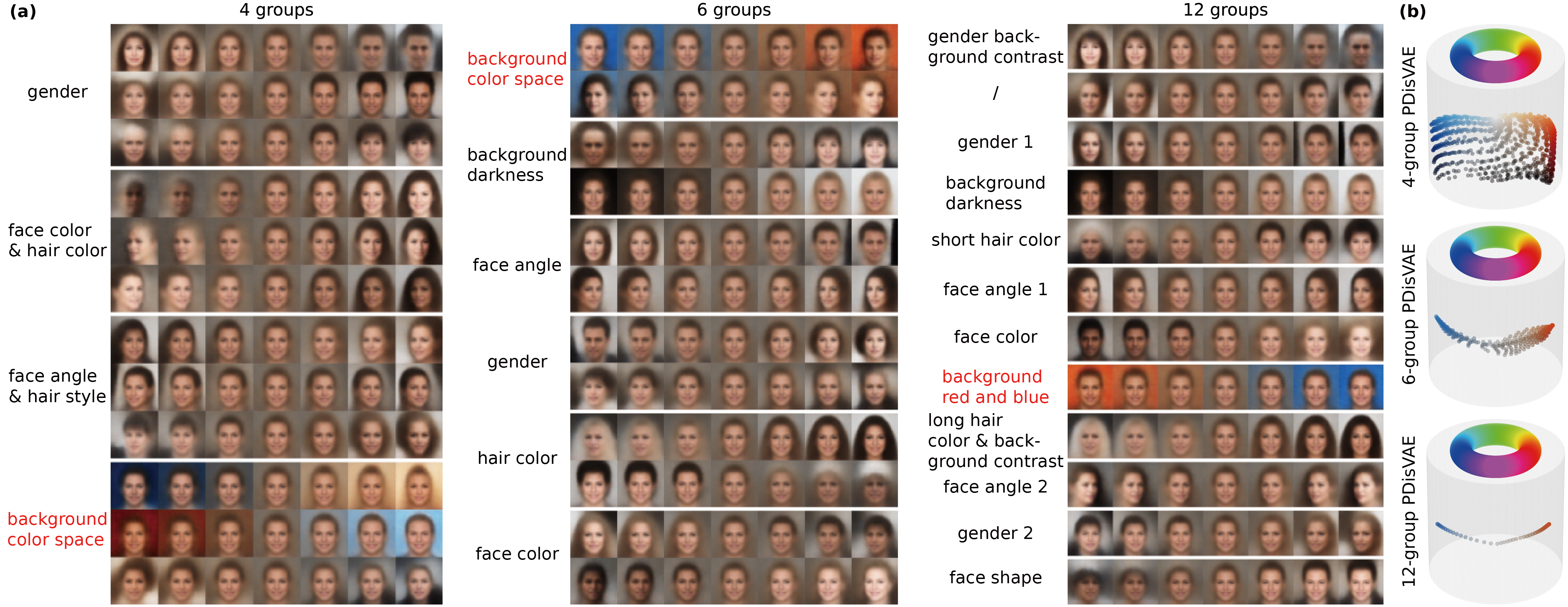}
    \vspace{-0.3in}
    \caption{\textbf{(a)}: Reconstructed images are shown by varying one of the $K=12$ latent dimensions from PDisVAE applied to the CelebA dataset, with different numbers of groups $G\in\{4, 6, 12\}$. Each row corresponds to varying one latent component (dimension) while fixing all others to 0s. \textbf{(b)} The spanned color space by the red-annotated color group in the $\cbr{4, 6, 12}$-group PDisVAE.}
    \label{fig:celeba}
\end{figure*}

\subsection{Synthetic application: partial dsprites}\label{sec:pdsprites}
\paragraph{Dataset.} To understand the application scenario of PDisVAE, we created a synthetic dataset called partial dsprites (pdsprites), inspired by \citet{dsprites17}. Unlike the original dsprites, which features six fully independent latent dimensions, we only keep three latent components: $x$-location ($z_1$), $y$-location ($z_2$), and size ($z_3$), where $x$ and $y$ locations are entangled (not independent) with each other while this group is independent to the size, i.e., $(z_1, z_2) \perp z_3$. The generating process is depicted in Fig.~\ref{fig:pdsprites}(a), resulting in 805 gray-scaled images of shape $32 \times 32$.

\noindent\textbf{Experimental setup.} For each method, we use Adam to train a deep CNN VAE \citep{burgess2018understanding} for 5,000 epochs with a learning rate of $1\times 10^{-3}$. For $\beta$-TCVAE and PDisVAE, the TC/PC coefficient is set as $\beta = 4$. Given the true latent is $(z_1, z_2) \perp z_3$, learning two rank-2 groups ($K = 4 = G \times H = 2 \times 2$) should be able to find one group representing the location of the square and another rank-deficient group (contains a dummy latent component) representing the size of the square. Note that this setup is a model mismatch case, as we do not know the exact observation generating function $\bm f$; we only understand the semantic relationship between $\bm z$ and $\bm x$.

\paragraph{Results.} Fig.~\ref{fig:pdsprites_latent_within_group} shows the estimated latent from all methods after alignment. PDisVAE has the highest latent $R^2$ and the second lowest PC. Notably, PDisVAE successfully discovers two empty areas in the upper and lower gray triangular regions in group 1, reflecting the true latent distribution depicted in Fig.~\ref{fig:pdsprites}(a). Additionally, PDisVAE captures leveled size scales in $z_3$, showing smaller sizes for smaller $z_3$ and larger sizes for larger $z_3$, making it the closest representation of the true $z_3$ compared to other methods. Appendix.~\ref{appendix:pdsprites} contains more plots and quantitative comparisons.

Fig.~\ref{fig:pdsprites}(b) shows the reconstructed images by varying each of the two groups found by $\beta$-BTCVAE and PDisVAE, respectively. Group 1 from PDisVAE represents the location, with an empty center due to fewer observation samples in that area (see the region around $(z_1, z_2) = (0, 0)$ in Fig.~\ref{fig:pdsprites}(a)). Besides, the square is expected to not appear in the top middle or bottom middle of the image, since there is no observation in the dataset that appears in those regions. The size is embedded in group 2, roughly along the $\hat z_4$ direction. In contrast, $\beta$-TCVAE mixes size and location in both groups because it enforces independence across all four components, which is incompatible with the fact that two location components are entangled together and independent of the third size component.

\subsection{Real-world applications}\label{sec:real_world}
To evaluate the performance and flexibility of PDisVAE in real-world applications, we train it on two real-world datasets, described in the following paragraphs. Since the true latent structure is unknown in these cases, we experiment with different group configurations for PDisVAE. Note that when $G=1$, $\operatorname{PC}\equiv 0$ and PDisVAE reduces to the standard VAE, and when $G=K$, PDisVAE reduces to the fully disentangled VAE, e.g., $\beta$-TCVAE or FactorVAE.

\begin{figure*}[!ht]
    \centering
    \includegraphics[width=\linewidth]{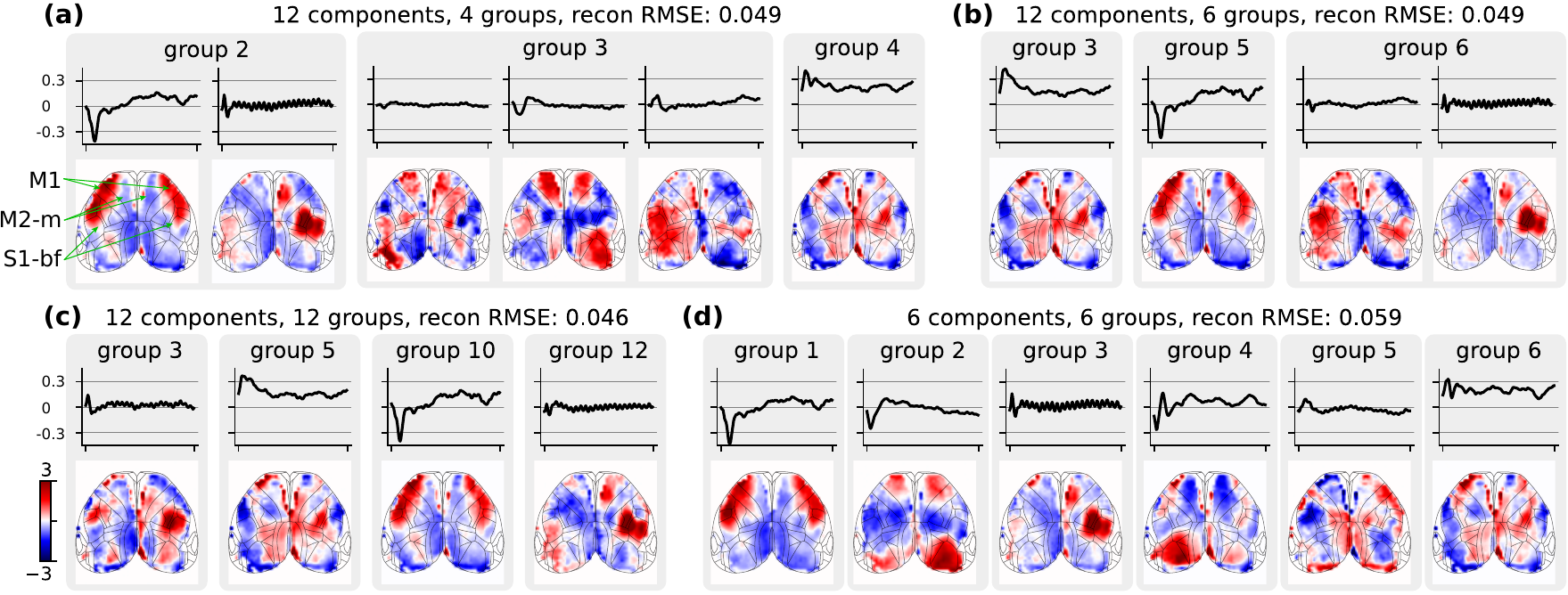}
    \vspace{-0.3in}
    \caption{Brain maps $\cbr{\bm z_g^n}_{n=1}^{50\times 50}$ and the corresponding time series $\bm A_{:, g}$ from the learned groups by different PDisVAE configurations $(K,G)$, i.e., $K$ components, $G$ groups, and the group rank is $H = K/G$. Some groups contain dummy dimensions, so the effective group rank is lower than the specified group rank, and hence we only show those effective components.}
    \label{fig:voltage}
\end{figure*}

\paragraph{CelebaA.} The dataset contains 202,599 face images \citep{liu2015faceattributes}, cropped and rescaled to $(3, 64, 64)$. The encoder and decoder are deep CNN-based image-nets \citep{burgess2018understanding}. We fix the latent dimensionality $K=12$ and vary the number of groups $G\in\cbr{1,2,3,4,6,12}$. Training settings are similar to the previous experiments.

Fig.~\ref{fig:celeba}(a) shows the reconstructed images by varying each of the $K=12$ components while fixing others as zero, for $G\in\cbr{4,6,12}$. The group meanings are annotated on the left. Particularly, with 4 or 6 groups, some attributes are represented by a group of higher rank rather than a single latent component, such as background color. Certain attributes are dependent on each other represented by a group, like the face color \& hair color in the $G=4$ setting. These important interpretations are harder to find by the fully disentangled $G=12$ setting. Besides, fully disentangled VAE may fail to ensure perfect independence if the component setting and the true latent factor are largely mismatched (which is also hard to determine), like gender 1 and gender 2 in the $G=12$ setting.

To understand how one semantic attribute is represented by multiple components within a group, we use background color as an example. The $G=12$ groups setting in Fig.~\ref{fig:celeba}(a) shows that the background color is represented by a single component, which restricts the expression to a 1D color manifold as shown in $G=12$ HSV cylinder in Fig.~\ref{fig:celeba}(b), which is not reasonable. With multiple latent components in a group representing background color, the background color can be expressed in 2D or 3D color manifolds as shown in $G=6$ and $G=4$ HSV cylinders, offering a more expressive and realistic representation. Results from all group settings are displayed in Fig.~\ref{fig:celeba_all} in Appendix.~\ref{appendix:supplimentary_results}.

\paragraph{Mouse dorsal cortex voltage imaging.} The dataset used in this study is a trial-averaged voltage imaging (method by \cite{lu2023widefield}) sequence from a mouse collected by us. It comprises 150 frames of $50 \times 50$ dorsal cortex voltage images, recorded while the mouse was subjected to a left-side air puff stimulus lasting 0.75 seconds. Each pixel is treated as a sample, and a linear model $\bm{x} \sim \Nc(\bm{A}\bm{z}, \bm{\sigma}^2\bm{I})$ is learned. We investigate different numbers of groups $G \in \{1, 2, 3, 4, 6, 12\}$ while keeping the number of components constant at $K = 12$. Additionally, we explore fully disentangled models by varying $K \in \{1, 2, 3, 4, 6, 12\}$ with $G = K$. The training procedures are similar to the previous experiments (see code for details).

Figure~\ref{fig:voltage} shows the brain maps and corresponding time series learned from various PDisVAE configurations $(K, G)$. Learning $K=12$ components with different $G$ groups (Fig.~\ref{fig:voltage}(a,b,c)) yields similar reconstruction RMSEs ($\approx 0.47$), but results in different latent representations.  Assuming $G=12$ as a fully disentangled model (Fig.~\ref{fig:voltage}(c)) is overly restrictive, as both group 3 and group 12 contain oscillations in the right primary somatosensory cortex-barrel field (S1-bf) and secondary motor cortex-medial (M2-m), demonstrating a lack of independence between these components. This configuration implies that there are not 12 independent components within this neural data. Conversely, assuming $G=4$ groups (Fig.~\ref{fig:voltage}(a)) is insufficient, as group 2 mixes not only the oscillatory signals right S1-bf and M2-m but also signals from other regions like the right primary motor cortex (M1). This implies a failure to capture the complete scope of independence in the data. A $G=6$ grouping (Fig.~\ref{fig:voltage}(b)) presents a more balanced approach. This model consists of six independent groups, each expressed by two latent components. Specifically, group 3’s S1-bf and M2-m remain active, indicating these areas are stimulated during the air puff; group 6 is primarily responsible for the oscillations in S1-bf and M2-m, with minimal interference from the M1 signal. Moreover, the brain maps in group 2 from the 4-group configuration are effectively delineated into groups 5 and 6 in the 6-group configuration, further affirming the relative independence of M1 from S1-bf and M2-m during stimulus exposure. The fully independent model with $(K, G) = (6, 6)$ (Fig.~\ref{fig:voltage}(d)) indicates that two components per group are necessary for accurate reconstruction. Specifically, having only one component per group is insufficient to reconstruct the raw video, as the RMSE for $(6, 6)$ is 0.059, which is significantly higher than the 0.049 RMSE for $(12, 6)$. The group reconstruction videos in the supplementary materials offer a more intuitive illustration of the full contribution of each group.

\section{Discussion}
In this work, we develop the partially disentangled variational auto-encoder (PDisVAE), a more flexible approach to handling group independence (partial disentanglement) in data, which is often a more realistic assumption than full independence (fully disentanglement) in a lot of applications.

\subsection{More discussions about interpreting semantic vs. statistical independence in practical applications}

In a lot of practical applications, we need to differentiate two concepts: semantic meaning vs. statistically independent group. It is possible that an independent group contains more than one semantic meaning. In the CelabA dataset, for example, it is likely that females have more warm backgrounds and males have more cold backgrounds. In this case, the background warm/cold is entangled with gender. In this case, we cannot separate these two semantic meanings since they are statistically dependent/entangled.

In our example of background color, especially Fig.~\ref{fig:celeba}(b), we interpret a group as background color based on our human understanding. However, we cannot rigorously prove that the background color is totally independent of the tiny facial feature changes. This is actually an important point we want to stress in this paper, like in Sec.~1 paragraph 2, Fig.~\ref{fig:pdsprites}(a), and Fig.~\ref{fig:celeba}(b). We can summarize the following four possibilities:\\
$\bullet$ one semantic meaning corresponds to one latent component (fully independent);\\
$\bullet$ one semantic meaning corresponds to several entangled latent components (a latent group);\\
$\bullet$ several semantic meanings correspond to one latent component (semantic meanings are entangled and encoded by one latent component);\\
$\bullet$ several semantic meanings correspond to several latent components (semantic meanings are entangled and encoded by several latent components).\\
This is the key reason we generalize fully disentangled VAE to partially disentangled VAE (PDisVAE) since PDisVAE considers all these possibilities that exist in nearly all real-world datasets (maybe with the probability of 1). We view this as our paper's key take-home message that we really need to jump out of the stereotype that one latent component should correspond to one semantic meaning.

For example, in the partial dsprites (pdsprites) dataset shown in Fig.~\ref{fig:pdsprites}(a), although we humans think $x$ location and $y$ location are two separable semantic meanings, they are statistically dependent/entangled with each other, so we cannot separate them but put them in one group, and that is why fully disentangled VAEs (e.g., $\beta$-TCVAE) fails with this dataset (Fig.~\ref{fig:pdsprites}(b)). We can think $x$ and $y$ as two semantic meanings or say $(x,y)$ "location" is one semantic meaning, but the ground truth is that $x$ location and $y$ location are entangled, not statistically separable, and hence should be encoded by a latent group of at least rank-2.

A similar reason also holds for the color distribution we plot in Fig.~\ref{fig:celeba}(b). If we use a fully disentangled VAE, we can only interpret that the background color (from red to blue, a curve in HSV space) is encoded by one latent component, but that might not be the fact. We do show in Fig. 9(b) that with more latent components entangled with each other as a group, the background color semantic meaning can be expressed more fully (a 2D manifold or a restricted 3D region that is not evenly distributed).

Therefore, no one can promise an absolutely perfect correspondence between semantic meaning(s) and a latent component/group. All researchers can do is validate the correctness of their method on synthetic datasets, as we do in Sec. 4.1, and get more interpretable (but cannot promise perfect correspondence) disentanglement results on real-world datasets. Generally speaking, it is nearly impossible for all kinds of disentangling methods to find pure correspondence between a latent component/group and one semantic meaning on real-world datasets. At least there are some noises including other semantic meanings of tiny magnitude. This kind of result should be acceptable in the field of representational learning (disentanglement), especially on real-world datasets where there is no true latent. Otherwise, any interpretation from any method could have small flaws (that can even come from random seeds or the floating point precision of the training device).

\subsection{Benefits and limitations}
PDisVAE is a generalized method, which naturally reduces to standard VAE and fully disentangled VAE, by setting the number of groups to 1 or equal to the latent dimensionality. PDisVAE allows the existence of dummy latent components in groups if the number of learned latent components is less than the specified group rank.

A potential limitation of PDisVAE is the need for an adequate number of groups and components to accurately express the disentangled latent space, expecially when the data demands it, but we may not have guidance on this information. To address this, we might either try different configurations or develop techniques for automatic group rank reduction during training to enhance the performance, which presents a promising direction for further work.







\section*{Impact Statement}
This paper presents work whose goal is to advance the field of 
Machine Learning. There are many potential societal consequences 
of our work, none which we feel must be specifically highlighted here.


\bibliography{ref}

\begin{thebibliography}{35}
\providecommand{\natexlab}[1]{#1}
\providecommand{\url}[1]{\texttt{#1}}
\expandafter\ifx\csname urlstyle\endcsname\relax
  \providecommand{\doi}[1]{doi: #1}\else
  \providecommand{\doi}{doi: \begingroup \urlstyle{rm}\Url}\fi

\bibitem[Achille \& Soatto(2017)Achille and Soatto]{achille2017emergence}
Achille, A. and Soatto, S.
\newblock On the emergence of invariance and disentangling in deep representations.
\newblock \emph{CoRR}, 2017.

\bibitem[Ahuja et~al.(2022)Ahuja, Hartford, and Bengio]{ahuja2022weakly}
Ahuja, K., Hartford, J.~S., and Bengio, Y.
\newblock Weakly supervised representation learning with sparse perturbations.
\newblock \emph{Advances in Neural Information Processing Systems}, 35:\penalty0 15516--15528, 2022.

\bibitem[Alemi et~al.(2016)Alemi, Fischer, Dillon, and Murphy]{alemi2016deep}
Alemi, A.~A., Fischer, I., Dillon, J.~V., and Murphy, K.
\newblock Deep variational information bottleneck.
\newblock \emph{arXiv preprint arXiv:1612.00410}, 2016.

\bibitem[Bengio et~al.(2013)Bengio, Courville, and Vincent]{bengio2013representation}
Bengio, Y., Courville, A., and Vincent, P.
\newblock Representation learning: A review and new perspectives.
\newblock \emph{IEEE transactions on pattern analysis and machine intelligence}, 35\penalty0 (8):\penalty0 1798--1828, 2013.

\bibitem[Bhowal et~al.()Bhowal, Soni, and Rambhatla]{bhowalvariational}
Bhowal, P., Soni, A., and Rambhatla, S.
\newblock Why do variational autoencoders really promote disentanglement?
\newblock In \emph{Forty-first International Conference on Machine Learning}.

\bibitem[Blei et~al.(2017)Blei, Kucukelbir, and McAuliffe]{blei2017variational}
Blei, D.~M., Kucukelbir, A., and McAuliffe, J.~D.
\newblock Variational inference: A review for statisticians.
\newblock \emph{Journal of the American statistical Association}, 112\penalty0 (518):\penalty0 859--877, 2017.

\bibitem[Burgess et~al.(2018)Burgess, Higgins, Pal, Matthey, Watters, Desjardins, and Lerchner]{burgess2018understanding}
Burgess, C.~P., Higgins, I., Pal, A., Matthey, L., Watters, N., Desjardins, G., and Lerchner, A.
\newblock Understanding disentangling in $\beta$-vae.
\newblock \emph{arXiv preprint arXiv:1804.03599}, 2018.

\bibitem[Calhoun et~al.(2009)Calhoun, Liu, and Adal{\i}]{calhoun2009review}
Calhoun, V.~D., Liu, J., and Adal{\i}, T.
\newblock A review of group ica for fmri data and ica for joint inference of imaging, genetic, and erp data.
\newblock \emph{Neuroimage}, 45\penalty0 (1):\penalty0 S163--S172, 2009.

\bibitem[Chen et~al.(2018)Chen, Li, Grosse, and Duvenaud]{chen2018isolating}
Chen, R.~T., Li, X., Grosse, R.~B., and Duvenaud, D.~K.
\newblock Isolating sources of disentanglement in variational autoencoders.
\newblock \emph{Advances in neural information processing systems}, 31, 2018.

\bibitem[Dubois et~al.(2019)Dubois, Kastanos, Lines, and Melman]{dubois2019dvae}
Dubois, Y., Kastanos, A., Lines, D., and Melman, B.
\newblock Disentangling vae.
\newblock \url{http://github.com/YannDubs/disentangling-vae/}, march 2019.

\bibitem[Fern{\'a}ndez et~al.(1995)Fern{\'a}ndez, Osiewalski, and Steel]{fernandez1995modeling}
Fern{\'a}ndez, C., Osiewalski, J., and Steel, M.~F.
\newblock Modeling and inference with $\upsilon$-spherical distributions.
\newblock \emph{Journal of the American Statistical Association}, 90\penalty0 (432):\penalty0 1331--1340, 1995.

\bibitem[Higgins et~al.(2017)Higgins, Matthey, Pal, Burgess, Glorot, Botvinick, Mohamed, and Lerchner]{higgins2017beta}
Higgins, I., Matthey, L., Pal, A., Burgess, C.~P., Glorot, X., Botvinick, M.~M., Mohamed, S., and Lerchner, A.
\newblock beta-vae: Learning basic visual concepts with a constrained variational framework.
\newblock \emph{ICLR (Poster)}, 3, 2017.

\bibitem[Hsu et~al.(2024)Hsu, Dorrell, Whittington, Wu, and Finn]{hsu2024disentanglement}
Hsu, K., Dorrell, W., Whittington, J., Wu, J., and Finn, C.
\newblock Disentanglement via latent quantization.
\newblock \emph{Advances in Neural Information Processing Systems}, 36, 2024.

\bibitem[Hyvarinen \& Morioka(2017)Hyvarinen and Morioka]{hyvarinen2017nonlinear}
Hyvarinen, A. and Morioka, H.
\newblock Nonlinear ica of temporally dependent stationary sources.
\newblock In \emph{Artificial Intelligence and Statistics}, pp.\  460--469. PMLR, 2017.

\bibitem[Hyv{\"a}rinen \& Oja(2000)Hyv{\"a}rinen and Oja]{hyvarinen2000independent}
Hyv{\"a}rinen, A. and Oja, E.
\newblock Independent component analysis: algorithms and applications.
\newblock \emph{Neural networks}, 13\penalty0 (4-5):\penalty0 411--430, 2000.

\bibitem[Hyv{\"a}rinen et~al.(2009)Hyv{\"a}rinen, Hurri, Hoyer, Hyv{\"a}rinen, Hurri, and Hoyer]{hyvarinen2009independent}
Hyv{\"a}rinen, A., Hurri, J., Hoyer, P.~O., Hyv{\"a}rinen, A., Hurri, J., and Hoyer, P.~O.
\newblock \emph{Independent component analysis}.
\newblock Springer, 2009.

\bibitem[Khemakhem et~al.(2020)Khemakhem, Kingma, Monti, and Hyvarinen]{khemakhem2020variational}
Khemakhem, I., Kingma, D., Monti, R., and Hyvarinen, A.
\newblock Variational autoencoders and nonlinear ica: A unifying framework.
\newblock In \emph{International conference on artificial intelligence and statistics}, pp.\  2207--2217. PMLR, 2020.

\bibitem[Kim \& Mnih(2018)Kim and Mnih]{kim2018disentangling}
Kim, H. and Mnih, A.
\newblock Disentangling by factorising.
\newblock In \emph{International conference on machine learning}, pp.\  2649--2658. PMLR, 2018.

\bibitem[Kingma(2013)]{kingma2013auto}
Kingma, D.~P.
\newblock Auto-encoding variational bayes.
\newblock \emph{arXiv preprint arXiv:1312.6114}, 2013.

\bibitem[Kingma(2014)]{kingma2014adam}
Kingma, D.~P.
\newblock Adam: A method for stochastic optimization.
\newblock \emph{arXiv preprint arXiv:1412.6980}, 2014.

\bibitem[Kraskov et~al.(2004)Kraskov, St{\"o}gbauer, and Grassberger]{kraskov2004estimating}
Kraskov, A., St{\"o}gbauer, H., and Grassberger, P.
\newblock Estimating mutual information.
\newblock \emph{Physical Review E—Statistical, Nonlinear, and Soft Matter Physics}, 69\penalty0 (6):\penalty0 066138, 2004.

\bibitem[Lake et~al.(2017)Lake, Ullman, Tenenbaum, and Gershman]{lake2017building}
Lake, B.~M., Ullman, T.~D., Tenenbaum, J.~B., and Gershman, S.~J.
\newblock Building machines that learn and think like people.
\newblock \emph{Behavioral and brain sciences}, 40:\penalty0 e253, 2017.

\bibitem[Liu et~al.(2015)Liu, Luo, Wang, and Tang]{liu2015faceattributes}
Liu, Z., Luo, P., Wang, X., and Tang, X.
\newblock Deep learning face attributes in the wild.
\newblock In \emph{Proceedings of International Conference on Computer Vision (ICCV)}, December 2015.

\bibitem[Locatello et~al.(2019)Locatello, Bauer, Lucic, Raetsch, Gelly, Sch{\"o}lkopf, and Bachem]{locatello2019challenging}
Locatello, F., Bauer, S., Lucic, M., Raetsch, G., Gelly, S., Sch{\"o}lkopf, B., and Bachem, O.
\newblock Challenging common assumptions in the unsupervised learning of disentangled representations.
\newblock In \emph{international conference on machine learning}, pp.\  4114--4124. PMLR, 2019.

\bibitem[Lu et~al.(2023)Lu, Wang, Liu, Gou, Jaeger, and St-Pierre]{lu2023widefield}
Lu, X., Wang, Y., Liu, Z., Gou, Y., Jaeger, D., and St-Pierre, F.
\newblock Widefield imaging of rapid pan-cortical voltage dynamics with an indicator evolved for one-photon microscopy.
\newblock \emph{Nature Communications}, 14\penalty0 (1):\penalty0 6423, 2023.

\bibitem[Makhzani et~al.(2015)Makhzani, Shlens, Jaitly, Goodfellow, and Frey]{makhzani2015adversarial}
Makhzani, A., Shlens, J., Jaitly, N., Goodfellow, I., and Frey, B.
\newblock Adversarial autoencoders.
\newblock \emph{arXiv preprint arXiv:1511.05644}, 2015.

\bibitem[Matthey et~al.(2017)Matthey, Higgins, Hassabis, and Lerchner]{dsprites17}
Matthey, L., Higgins, I., Hassabis, D., and Lerchner, A.
\newblock dsprites: Disentanglement testing sprites dataset.
\newblock https://github.com/deepmind/dsprites-dataset/, 2017.

\bibitem[Meo et~al.(2024)Meo, Mahon, Goyal, and Dauwels]{meo2024alpha}
Meo, C., Mahon, L., Goyal, A., and Dauwels, J.
\newblock $\backslash alpha$ tc-vae: On the relationship between disentanglement and diversity.
\newblock In \emph{The Twelfth International Conference on Learning Representations}, 2024.

\bibitem[Schmidhuber(1992)]{schmidhuber1992learning}
Schmidhuber, J.
\newblock Learning factorial codes by predictability minimization.
\newblock \emph{Neural computation}, 4\penalty0 (6):\penalty0 863--879, 1992.

\bibitem[Sinz \& Bethge(2010)Sinz and Bethge]{sinz2010lp}
Sinz, F. and Bethge, M.
\newblock Lp-nested symmetric distributions.
\newblock \emph{The Journal of Machine Learning Research}, 11:\penalty0 3409--3451, 2010.

\bibitem[Sorrenson et~al.(2020)Sorrenson, Rother, and K{\"o}the]{sorrenson2020disentanglement}
Sorrenson, P., Rother, C., and K{\"o}the, U.
\newblock Disentanglement by nonlinear ica with general incompressible-flow networks (gin).
\newblock \emph{arXiv preprint arXiv:2001.04872}, 2020.

\bibitem[St{\"u}hmer et~al.(2020)St{\"u}hmer, Turner, and Nowozin]{stuhmer2020independent}
St{\"u}hmer, J., Turner, R., and Nowozin, S.
\newblock Independent subspace analysis for unsupervised learning of disentangled representations.
\newblock In \emph{International Conference on Artificial Intelligence and Statistics}, pp.\  1200--1210. PMLR, 2020.

\bibitem[Wang et~al.(2024)Wang, Li, Li, and Wu]{wang2024exploring}
Wang, Y., Li, C., Li, W., and Wu, A.
\newblock Exploring behavior-relevant and disentangled neural dynamics with generative diffusion models.
\newblock \emph{Advances in Neural Information Processing Systems}, 37, 2024.

\bibitem[Yang et~al.(2021)Yang, Liu, Chen, Shen, Hao, and Wang]{yang2021causalvae}
Yang, M., Liu, F., Chen, Z., Shen, X., Hao, J., and Wang, J.
\newblock Causalvae: Disentangled representation learning via neural structural causal models.
\newblock In \emph{Proceedings of the IEEE/CVF conference on computer vision and pattern recognition}, pp.\  9593--9602, 2021.

\bibitem[Zhou \& Wei(2020)Zhou and Wei]{zhou2020learning}
Zhou, D. and Wei, X.-X.
\newblock Learning identifiable and interpretable latent models of high-dimensional neural activity using pi-vae.
\newblock \emph{Advances in Neural Information Processing Systems}, 33:\penalty0 7234--7247, 2020.

\end{thebibliography}
\bibliographystyle{icml2025}

\newpage
\appendix
\onecolumn
\section{Appendix}
\subsection{Related works}\label{appendix:related_works}
Realizing there are a lot of methods related to latent disentanglement, we provide Tab.~\ref{tab:related_works} with a list to summarize their contributions and differences.

\begin{table}[!ht]
    \centering
    \caption{Related works}
    \begin{tabular}{lcc}
        \toprule
         & full disentanglement & partial disentanglement \\
        \midrule
        By prior (not flexible) & [1] & [4] \\
        By extra penalty to loss (flexible) & [2][3] & Our PDisVAE \\
        By auxiliary information (supervised) & [7] & \\
        Others & [5][6][8][9][10] & \\
        \bottomrule
    \end{tabular}
    \label{tab:related_works}
\end{table}

$\bullet$ [1] ICA \citep{hyvarinen2000independent}: Traditional ICA uses a non-Gaussian prior to achieving full disentanglement since independence is non-Gaussian from the statistical perspective. However, the choice of the non-Gaussian prior is critical and might be too rigid, hurting the flexibility of the method. \\
$\bullet$ [2] FactorVAE \citep{kim2018disentangling} [3] $\beta$-TCVAE \citep{chen2018isolating}: These two papers start from the statistical definition of full independence to add an extra total correlation to achieve full independence rigorously. The only difference between these two papers is their implementations of minimizing TC. \\
$\bullet$ [4] ISA-VAE \citep{stuhmer2020independent}: ISA-VAE realized the commonly existing group-wise independence (partial disentanglement) in the real-world data. It utilizes a group-wise independent prior called $L^p$-nested distribution to achieve the partial disentanglement. However, they did not validate their approach on partially disentangled synthetic datasets, but merely evaluated their approach using fully disentangled assumptions for dsprites and CelebA datasets. \\
$\bullet$ [5] $\beta$-VAE \citep{burgess2018understanding}: Directly penalize the KL divergence of the VAE ELBO loss, in which TC (in Eq.~\eqref{eq:decomposed_KL}) is implicitly penalized. This approach has been proven to be worse than $\beta$-VAE and FactorVAE. \\
$\bullet$ [6] \citep{locatello2019challenging}: This research presented common challenges in finding disentangled latent through an unsupervised approach, implying supervision with semantic latent labels might be necessary under the assumption of full latent disentanglement. This also gives us a hint that full disentanglement might be a strong and inappropriate assumption and could result in poor latent interpretation. \\
$\bullet$ [7] \citep{ahuja2022weakly}: This paper uses weak supervision from observations generated by sparse perturbations of the latent variables, which requires auxiliary information to the latent variables. \\
$\bullet$ [8] \citep{meo2024alpha}: This paper replace the traditional TC term with a novel TC lower bound to achieve not only disentanglement but generalized observation diversity. \\
$\bullet$ [9] \citep{bhowalvariational}: This paper claims that VAE with orthogonal structure could also achieve latent full disentanglement. \\
$\bullet$ [10] \citep{hsu2024disentanglement}: The full disentanglement is achieved by a technique called latent quantization. The approach is quantizing the latent space into discrete code vectors with a separate learnable scalar codebook per dimension. Besides, weight decay is also applied to the model regularization for better full disentanglement.

\clearpage
\subsection{Marginal independence}\label{appendix:marginal_independence}
This part explains the sufficient but not necessary relationship between ``group-wise independence'' and ``marginal independence''. Consider latent variable $\bm z\in\Rd^M$ contains $M$ components that are independent between $G$ groups. The formal expression is
\begin{equation}
    \bigperp_{g=1}^G \rbr{z_{(g-1)H+1},\dots,z_{gH}} \implies \bigwedge_{i\in g_1,j\in g_2,g_1\neq g_2} z_i \perp z_j
\end{equation}
but not vice versa. We start from the simple counterexample mentioned in Sec.~\ref{sec:problem_definition} to explain why group-wise independence is a sufficient but not necessary condition of marginal independence.

Consider three random variables $z_1, z_2, z_3$ that follow the joint distribution shown in Tab.~\ref{tab:counterexample}. Notice that $z_3$ is actually the exclusive or of the two others, i.e., $z_3 = \operatorname{XOR}(z_1, z_2)$. It is obvious that $z_3 \not\perp (z_1, z_2)$ since when $z_1$ and $z_2$ are different, $p(z_3|z_1, z_2)$ is a discrete Dirac delta function at $z_3 = 0$; but when $z_1$ and $z_2$ are the same, $p(z_3|z_1, z_2)$ is a discrete Dirac delta function at $z_3 = 1$. Marginally, however, $z_1 \perp z_3$ and $z_2 \perp z_3$, since $p(z_3|z_1)$ is always a $p=0.5$ Bernoulli distribution regardless of the value of $z_1$. The same arguments are also applicable to $z_2 \perp z_3$. Therefore, this counterexample shows that $z_1\perp z_3, z_2\perp z_3 \centernot \implies (z_1,z_2) \perp z_3$. In other words, marginal independence does not imply group-wise independence.

Another way of checking this example is by the following theorem. 
\begin{theorem}
    $(x_1, \dots, x_I) \perp (y_1, \dots, y_J) \iff \big(f(x_1,\dots, x_I) \perp g(y_1, \dots, y_J)$ $\forall$ measurable functions $f$ and $g\big)$.
\end{theorem}
\begin{proof}
    The $\implies$ is obvious. To prove $\impliedby$, simply taking $f$ and $g$ to be identity function, i.e., $f(x_1, \dots, x_I) = (x_1, \dots, x_I)$, $g(y_1, \dots, y_J) = (y_1, \dots, y_J)$.
\end{proof}
To check the example, consider the distribution of $(z_1 + z_2)$. $p(z_3 |(z_1 + z_2) = 0)$ is a discrete Dirac delta function at $z_3 = 1$, which is different from $p(z_3 | (z_1 + z_2) = 1)$ is a discrete Dirac delta function at $z_3 = 0$. Therefore, $(z_1, z_2) \not\perp z_3$.

To rigorously diagnose where $\impliedby$ breaks, we can write
\begin{equation}
    p(z_1, z_2, z_3) = p(z_1 | z_2, z_3) p(z_2, z_3) = p(z_1 | z_2, z_3) p(z_2) p(z_3). 
\end{equation}
Note that in the last term, $p(z_1|z_2, z_3) \neq p(z_1 | z_2)$. Specifically, $z_3$ cannot be removed just because of $z_1 \perp z_3$.

\begin{table}[htbp]
    \centering
    \caption{The distribution table of $p(z_1,z_2,z_3)$.}
    \label{tab:counterexample}
    \begin{tabular}{cccc}
        \toprule
        $z_1$ & $z_2$ & $z_3$ & $p(z_1,z_2,z_3)$ \\
        \midrule
        0 & 0 & 1 & 0.25 \\
        0 & 1 & 0 & 0.25 \\
        1 & 0 & 0 & 0.25 \\
        1 & 1 & 1 & 0.25 \\
        \bottomrule
    \end{tabular}
\end{table}

\subsection{Batch approximation}\label{appendix:batch_approximation}
\subsubsection{Importance sampling}
Although Eq.~\eqref{eq:IS} in the main text intuitively gives the batch approximation, we still need a rigorous derivation to prove this is exactly the importance sampling (IS) we want. First, we have the aggregated posterior that can be expressed in different ways:
\begin{equation}
    q(z) = \sum_{n=1}^N q(z,n) = \sum_{n=1}^N q(z|n) q(n) = \frac{1}{N} \sum_{n=1}^N q(z|n) = \Ed_{q(n)}[q(z|n)].
\end{equation}
However, to not confuse readers, we will keep the form $q(z) = \sum_{n=1}^N q(z,n)$ until the last step.

When we have a batch of size $M$: $\Bc_M \coloneqq \cbr{n_1,n_2,\dots, n_M}$ (without replacement) and a particular sampled $z\sim q(z|n_*)$, where $n_*\in\Bc_M$, we want the importance sampling approximation of $q(z)$. According to Monte Carlo estimation,
\begin{equation}
    \hat q(z) = \frac{1}{M}\sum_{m=1}^M \frac{q(z,n_m)}{r(n_m)},
\end{equation}
where $r$ is the proposal distribution. Note that $r(n_m) \neq \frac{1}{N},\ \forall n_m \in \Bc$, since we must have $n_*\in\Bc_M$. Therefore, we need to understand the distribution of $r(n_m)$.

First, since we must have $n_*\in\Bc_M$, and the Monte Carlo estimation is the average on $\Bc_M$,
\begin{equation}
    r(n_*) = \underbrace{1}_{n_* \text{ must be in } \Bc_M} \times \underbrace{\frac{1}{|\Bc_M|}}_{n_*\ \text{is a Monte Carlo sample from} \Bc_M} = \frac{1}{M}.
\end{equation}
Second, for other $n_m\notin \Bc_M$,
\begin{equation}
    r(n_m) = \underbrace{\frac{\binom{N-2}{M-2}}{\binom{N-1}{M-1}}}_{n_m \text{is selected in batch } \Bc_M} \times \underbrace{\frac{1}{|\Bc_M|}}_{n_m\ \text{is a Monte Carlo sample from} \Bc_M} = \frac{M-1}{N-1}\frac{1}{M}.
\end{equation}
$\binom{N-1}{M-1} = \frac{(N-1)!}{(M-1)!((N-1) - (M-1))!}$ is the number of all possible combinations of $\Bc_M$ that already contains $n_*$ (so we choose $M-1$ from the remaining $N-1$). $\binom{N-2}{M-2} = \frac{(N-2)!}{(M-2)!((N-2) - (M-2))!}$ is the number of all possible combinations of $\Bc_M$ that already contains $n_*$ and also contains $n_m$ (so we choose $M-2$ from the remaining $N-2$). Finally, we have
\begin{equation}
    \begin{split}
        \hat q(z) = & \frac{1}{M} \sum_{m=1}^M \frac{q(z,n_m)}{r(n_m)} \\
        = & \frac{1}{M} \frac{q(z|n_*) q(n_*)}{r(n_*)} + \sum_{n_m\in (\Bc_M\backslash\cbr{n_*})} \frac{1}{M} \frac{q(z|n_m)q(n_m)}{r(n_m)} \\
        = & \frac{1}{M} \frac{q(z|n_*) \frac{1}{N}}{\frac{1}{M}} + \sum_{n_m\in (\Bc_M\backslash\cbr{n_*})} \frac{1}{M} \frac{q(z|n_m)\frac{1}{N}}{\frac{M-1}{N-1}\frac{1}{M}} \\
        = & \frac{1}{N} q(z|n_*) + \sum_{n_m\in (\Bc_M\backslash\cbr{n_*})} \frac{N-1}{M-1}\frac{1}{N} q(z|n_m).
    \end{split}
\end{equation}

\subsubsection{Variance}
From \citet{chen2018isolating}, without loss of generality, assume $n_* = n_1$ and
\begin{equation}
    \begin{split}
        \mathrm{MSS} = & \frac{1}{N} q(z|n_*) + \sum_{m=2}^{M-1} \frac{1}{M-1} q(z|n_m) + \frac{N-M+1}{N(M-1)} q(z|n_M) \\
        = & \frac{1}{N} q(z|n_*) + \sum_{m=2}^{M-1} \frac{N}{M-1} \frac{1}{N}q(z|n_m) + \frac{N-M+1}{(M-1)} \frac{1}{N} q(z|n_M).
    \end{split}
\end{equation}

A sketch to compute the variances of the two methods is to think of them as sampled datasets of size $M$. Specifically, for IS, the inverse importance weights are a dataset of $\mathrm{IS}_0 \coloneqq \cbr{1, \underbrace{\frac{N-1}{M-1}, \dots, \frac{N-1}{M-1}}_{M-1}}$. For, MSS, the inverse importance weights are a dataset of $\mathrm{MSS}_0 \coloneqq \cbr{1, \underbrace{\frac{N}{M-1}, \dots, \frac{N}{M-1}}_{M-2}, \frac{N-M+1}{M-1}}$.

There means are all $\frac{N}{M}$, since
\begin{equation}
    \begin{cases}
        \overline{\mathrm{MSS}_0} = \frac{1}{M} \rbr{1 + (M-2) \frac{N}{M-1} + \frac{N-M+1}{M-1}} = \frac{N}{M} \\
        \overline{\mathrm{IS}_0} = \frac{1}{M} \rbr{1 + (M-1) \frac{N-1}{M-1}} = \frac{N}{M}
    \end{cases}
\end{equation}

Now we compute their variances.
\begin{equation}\label{eq:variance_MSS}
    \begin{split}
        \Var[\mathrm{MSS}] \propto & \Var[\mathrm{MSS}_0] \\
        = & \frac{1}{M} \sbr{\rbr{1-\frac{N}{M}}^2 + (M-2)\rbr{\frac{N}{M-1} - \frac{N}{M}}^2 + \rbr{\frac{N-M+1}{M-1} - \frac{N}{M}}^2} \\
        = & \frac{2M^2 - (2N+2)M + N^2}{M^2(M-1)}.
    \end{split}
\end{equation}
\begin{equation}\label{eq:variance_IS}
    \begin{split}
        \Var[\mathrm{IS}] \propto & \Var[\mathrm{IS}_0] \\
        = & \frac{1}{M} \sbr{\rbr{1 - \frac{N}{M}}^2 + (M-1)\rbr{\frac{N-1}{M-1} - \frac{N}{M}}^2} \\
        = & \frac{(N-M)^2}{M^2(M-1)}.
    \end{split}
\end{equation}
Since
\begin{equation}
     \Var[\mathrm{IS}_0] - \Var[\mathrm{MSS}_0] = \frac{2-M}{M(M-1)} \leqslant 0,\ \forall M \geqslant 2,
\end{equation}
the effectiveness of IS is higher, and hence IS is a more stable approximation than MSS.

\subsubsection{Empirical evaluation}
To validate the aforementioned superiority of our IS batch estimation method, we simulate a dataset consisting of 10 data points shown in Fig.~\ref{fig:batch_approximation}(left). Each time, we run the three batch approximation methods on a batch of three randomly sampled points. We repeat this 1000 times and show their empirical evaluations in Fig.~\ref{fig:batch_approximation}(right). Compared with the unbiased MWS estimator, MMS and IS are unbiased. Compared with MMS, the IS estimator has low empirical variance across 1000 repeats, which implies a more stable estimation.

\begin{figure}[!ht]
    \centering
    \includegraphics[width=\linewidth]{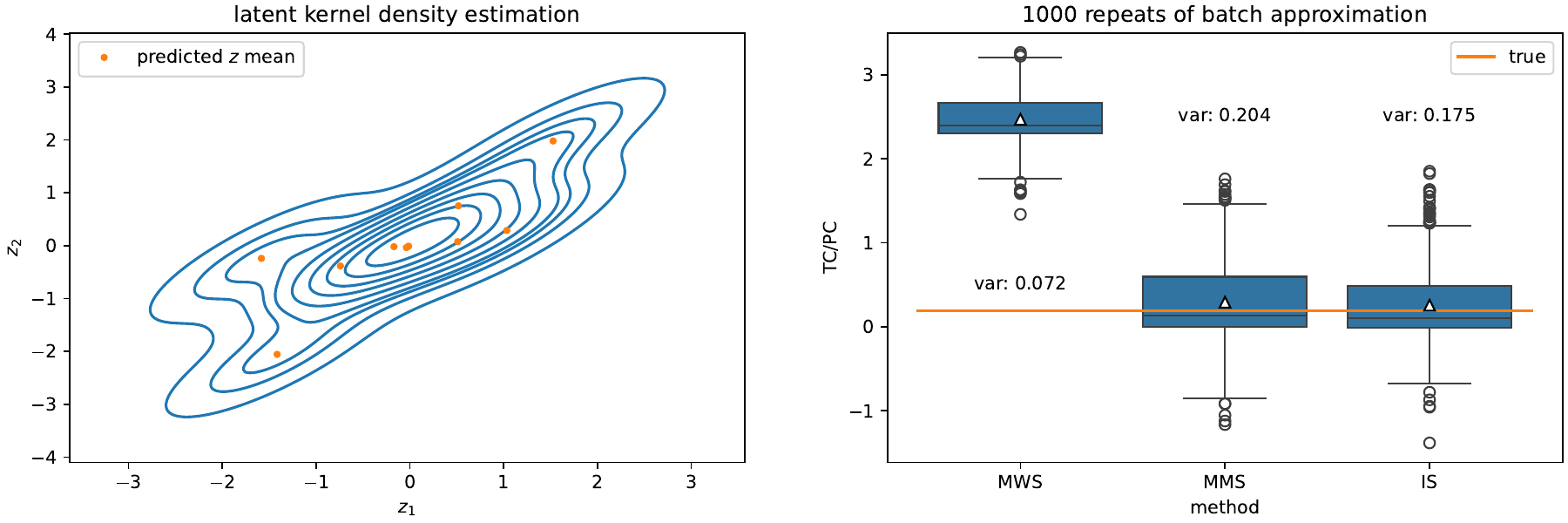}
    \vspace{-0.3in}
    \caption{\textbf{Left}: Predicted mean of the latent $\bm z = (z_1, z_2)$ and its kernel density estimation. \textbf{Right}: 1000 repeats of batch approximations by the three methods, their empirical variance across the 1000 repeats.}
    \label{fig:batch_approximation}
\end{figure}

\clearpage
\subsection{Supplementary results}\label{appendix:supplimentary_results}
\subsubsection{Synthetic validation: group-wise independent}
\begin{figure}[!ht]
    \centering
    \includegraphics[width=\linewidth]{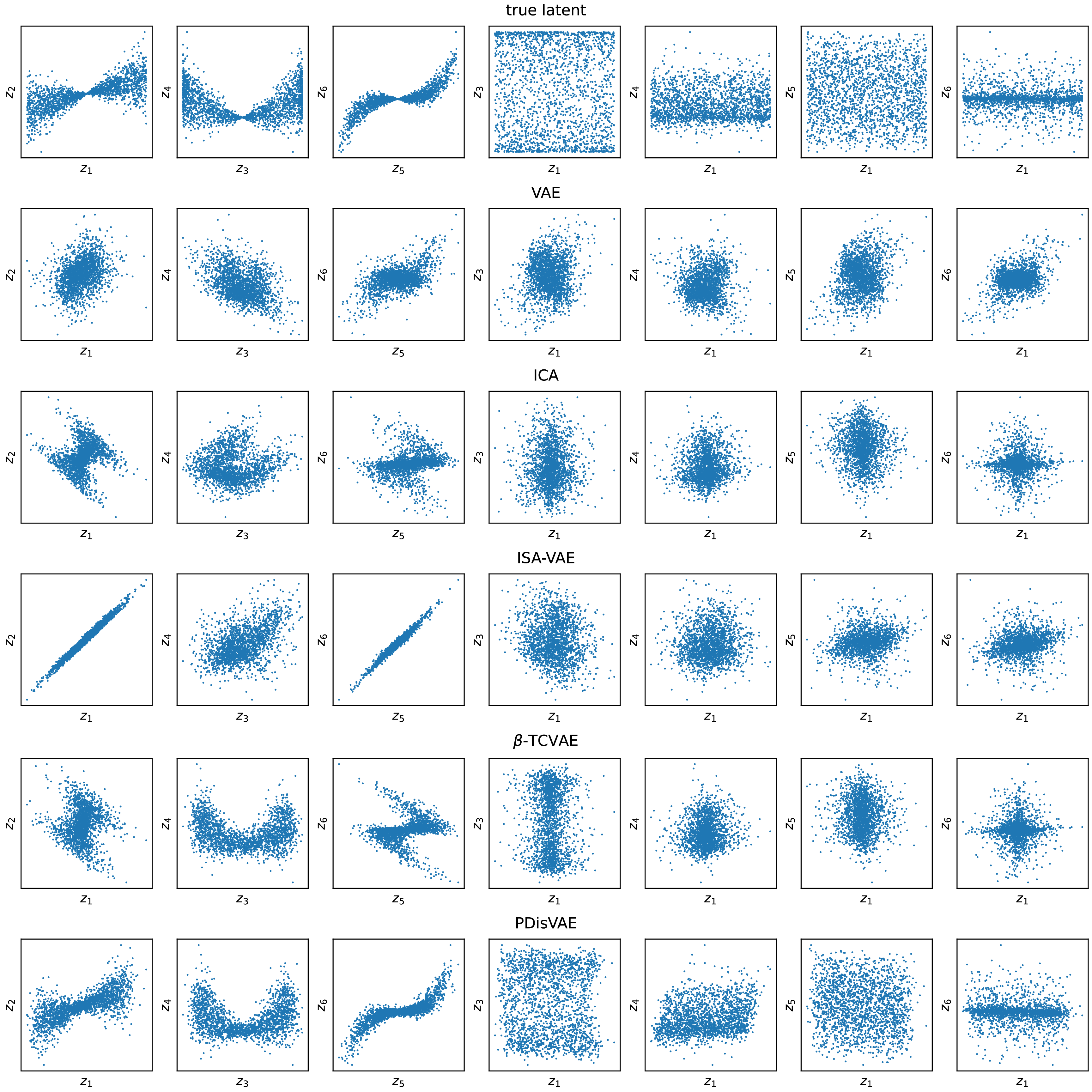}
    \vspace{-0.3in}
    \caption{Latent alignment results of different methods. Each group is aligned and matched to the true latent shown in Fig.~\ref{fig:synthetic_weak_box}(a). Some between-group pairs are also plotted to visually understand the marginally independent distributions between groups.}
    \label{fig:synthetic_weak_latent_all}
\end{figure}

\subsubsection{Ablation}\label{Appendix:ablation}

\subsubsection{Flexibly reduce to the fully independent case}\label{Appendix:fully_independent}
\paragraph{Dataset and experimental setup.} To validate that PDisVAE can get the same results as from a fully disentangled VAE when the latent is fully independent, we create a dataset consisting of $N=2000$ points in $K=3$ latent space $\bm z^{(n)} \in \Rd^3$, where the three latent components are independent with each other $z_1\perp z_2\perp z_3$. Their distributions are shown in Fig.~\ref{fig:synthetic_strong_box}(a) and Fig.~\ref{fig:synthetic_strong_latent}. The observation $\bm x$ is linearly mapped from the latent $\bm z$ to a $D = 20$ dimensional space $\bm x^{(n)} \in \Rd^{20}$, and then Gaussian noise $\epsilon_d^{(n)} \overset{i.i.d.}{\sim} \Nc\rbr{0,0.5^2}$ are added. Although we only have $K=3$ true latent components, we still learn $K=6$ components to compare their flexibility when the true number of latent components is unknown. The experimental setup is the same as the previous one.

\begin{figure*}[!ht]
    \centering
    \includegraphics[width=\linewidth]{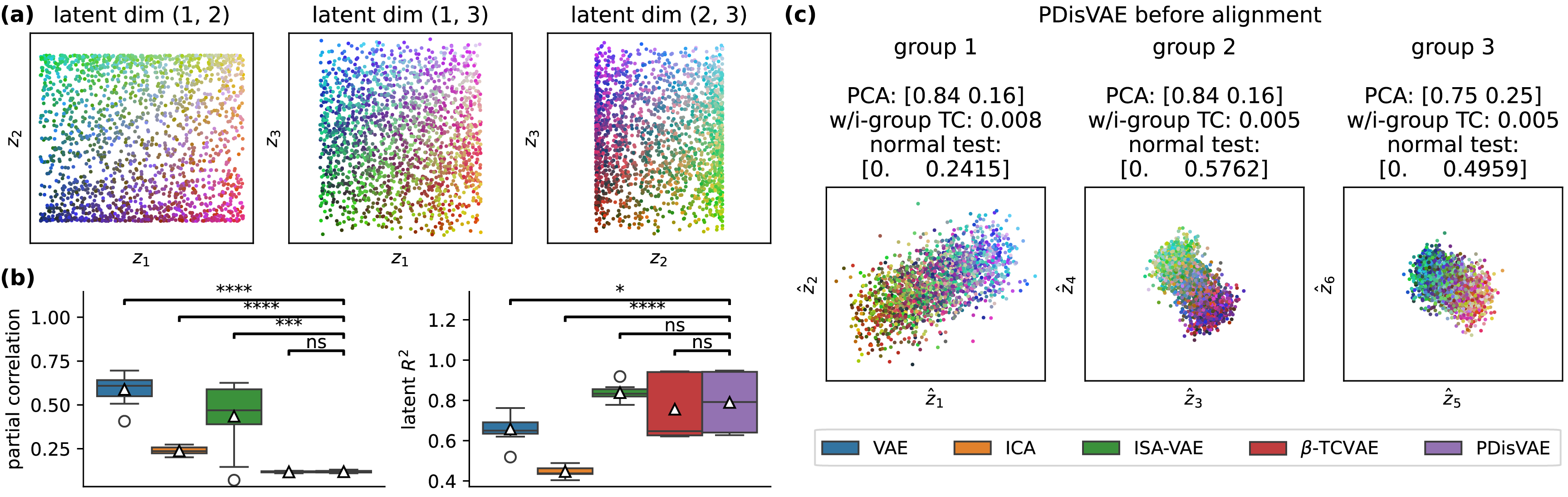}
    \vspace{-0.3in}
    \caption{\textbf{(a)}: The true latent $\bm z \in \Rd^3$ coded by RGB = $z_1z_2z_3$, where three components are $z_1\perp z_2\perp z_3$. \textbf{(b)}: The PC of the estimated latent and the latent $R^2$ after alignment to the true latent in (a). The $t$-test between PDisVAE and others shows that PDisVAE is similar to $\beta$-TCVAE (ns: $p>0.5$, *: $p\leqslant 0.05$, ****: $p\leqslant 0.0001$). \textbf{(c)}: The estimated latent of PDisVAE before aligning to the true latent shown in (a). The arrow in each plot shows the embedded true latent direction.}
    \label{fig:synthetic_strong_box}
\end{figure*}

\paragraph{Results.} The PC box plot and latent $R^2$ plot in Fig.~\ref{fig:synthetic_strong_box}(b) show that both $\beta$-TCVAE and PDisVAE achieve the lowest partial correlation and the highest latent $R^2$ on this fully disentangled dataset, which implies that PDisVAE automatically reduces to fully independent result if the group rank is deficient, as illustrated in case 2 in Fig.~\ref{fig:illustration}. In general, the actual group rank can be detected by PDisVAE and if the true group rank is less than the specified group dimensionality, dummy estimated latents will complemented in the corresponding group. Due to the strong requirement in ICA that tries to find logcosh-independent components but only three exist, ICA is not able to correctly identify three and find three dummy dimensions. This means logcosh might be too strong to allow the existence of dummy variables, which could be harmful when we do not know the true number of latent components. Fig.~\ref{fig:synthetic_strong_latent} also visually shows that $\beta$-TCVAE and PDisVAE accurately estimate the three latent distributions the best, which is consistent with the latent $R^2$ plot in Fig.~\ref{fig:synthetic_strong_box}(b).

To identify the three dummy latent dimensions complementing the three groups respectively through an unsupervised approach, we plot the PDisVAE result before alignment in Fig.~\ref{fig:synthetic_strong_box}(c). First, within-group TCs are all very small, indicating that the result is not the case 1 in Fig.~\ref{fig:illustration}. Since ``independence is non-Gaussian'', we can find a direction within each group that yields $p>0.05$, which accepts the null hypothesis of the normal test that a Gaussian noise dummy dimension exists, corresponding to case 2 in Fig.~\ref{fig:illustration}. The arrows in Fig.~\ref{fig:synthetic_strong_box}(c) also visually indicate the embedded true latent direction.

\begin{figure}[!ht]
    \centering
    \includegraphics[width=\linewidth]{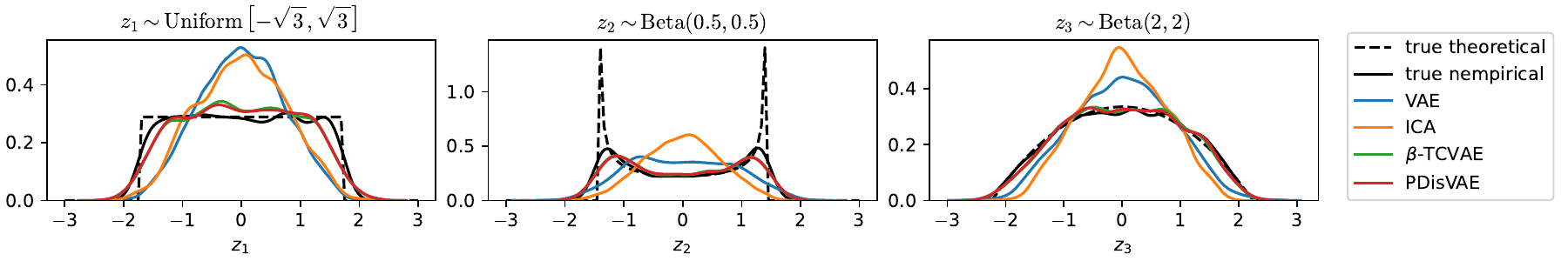}
    \vspace{-0.3in}
    \caption{Estimated and true latent distribution after alignment to the true latent shown in Fig.~\ref{fig:synthetic_strong_box}(a).}
    \label{fig:synthetic_strong_latent}
\end{figure}

\clearpage
\subsubsection{Synthetic application: partial dsprites}\label{appendix:pdsprites}
\begin{figure}[!ht]
    \centering
    \includegraphics[width=\linewidth]{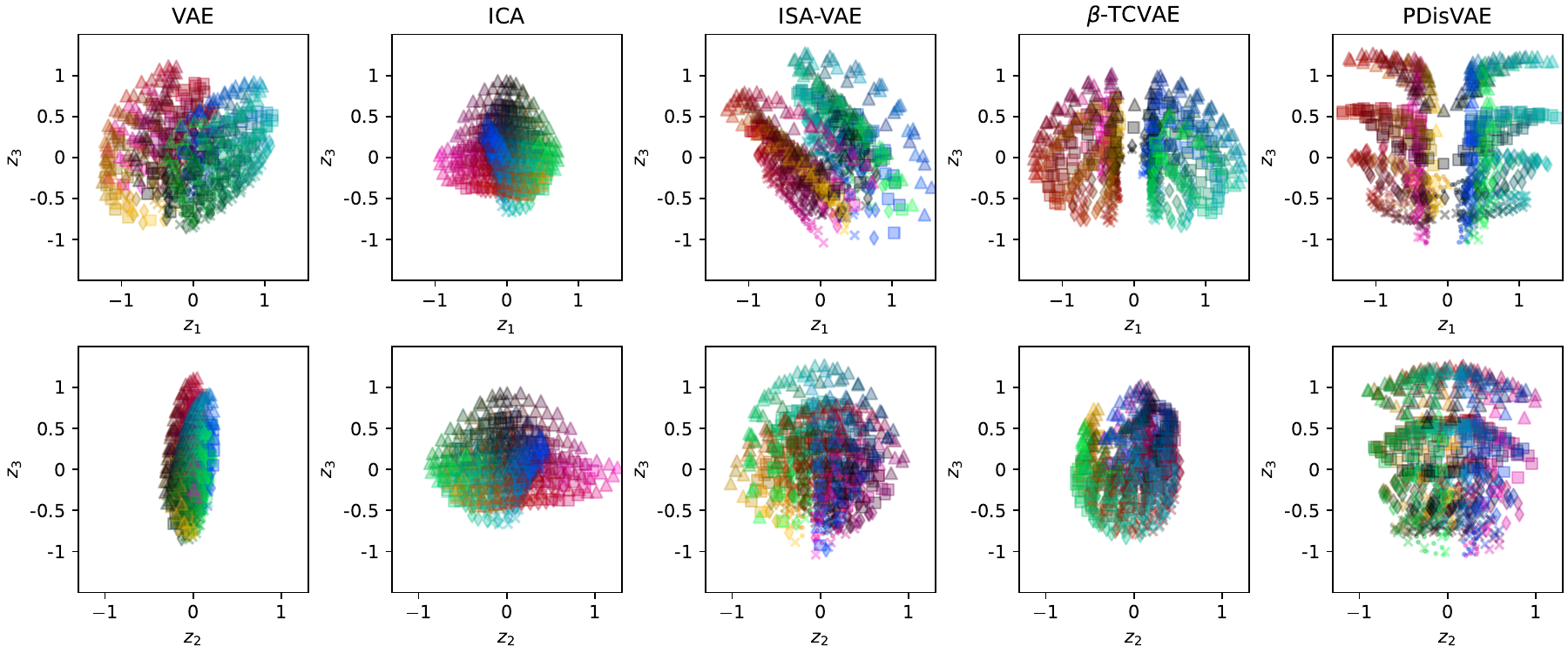}
    \vspace{-0.3in}
    \caption{The latent plot after alignment in latent space $(z_1, z_3)$ and $(z_2, z_3)$ for different methods. The color representation for location is the same as the color representation in Fig.~\ref{fig:pdsprites}(a), and the marker of the point in the latent plots represents the size of the square in the observation images.}
    \label{fig:pdsprites_latent_between_group}
\end{figure}

\begin{table}[!ht]
    \centering
    \caption{The PC, latent $R^2$, latent MSS, and adapted mutual information gap (MIG) evaluated for different methods on the dsprites dataset.}
    \label{tab:pdsprites}
    \begin{tabular}{lcccc}
        \toprule
         & PC $\downarrow$ & $R^2$ $\uparrow$ & MSE $\downarrow$ & MIG $\uparrow$ \\
        \midrule
        VAE & 1.01 (0.02) & 0.22 (0.04) & 0.29 (0.02) & 0.15 (0.01) \\
        ICA & 1.76 (0.07) & 0.22 (0.06) & 0.28 (0.03) & 0.14 (0.09) \\
        ISA-VAE & \textbf{0.70 (0.01)} & 0.23 (0.02) & 0.33 (0.01) & 0.24 (0.08) \\
        $\beta$-TCVAE & 0.91 (0.10) & 0.33 (0.06) &\textbf{0.24 (0.04)} & 0.36 (0.13) \\
        PDisVAE & \textbf{0.68 (0.04)} & \textbf{0.54 (0.08)} & \textbf{0.23 (0.04)} &\textbf{0.49 (0.07)} \\
        \bottomrule
    \end{tabular}
\end{table}

\clearpage
\subsubsection{Real-world applications}
\begin{figure}[!ht]
    \centering
    \includegraphics[width=\linewidth]{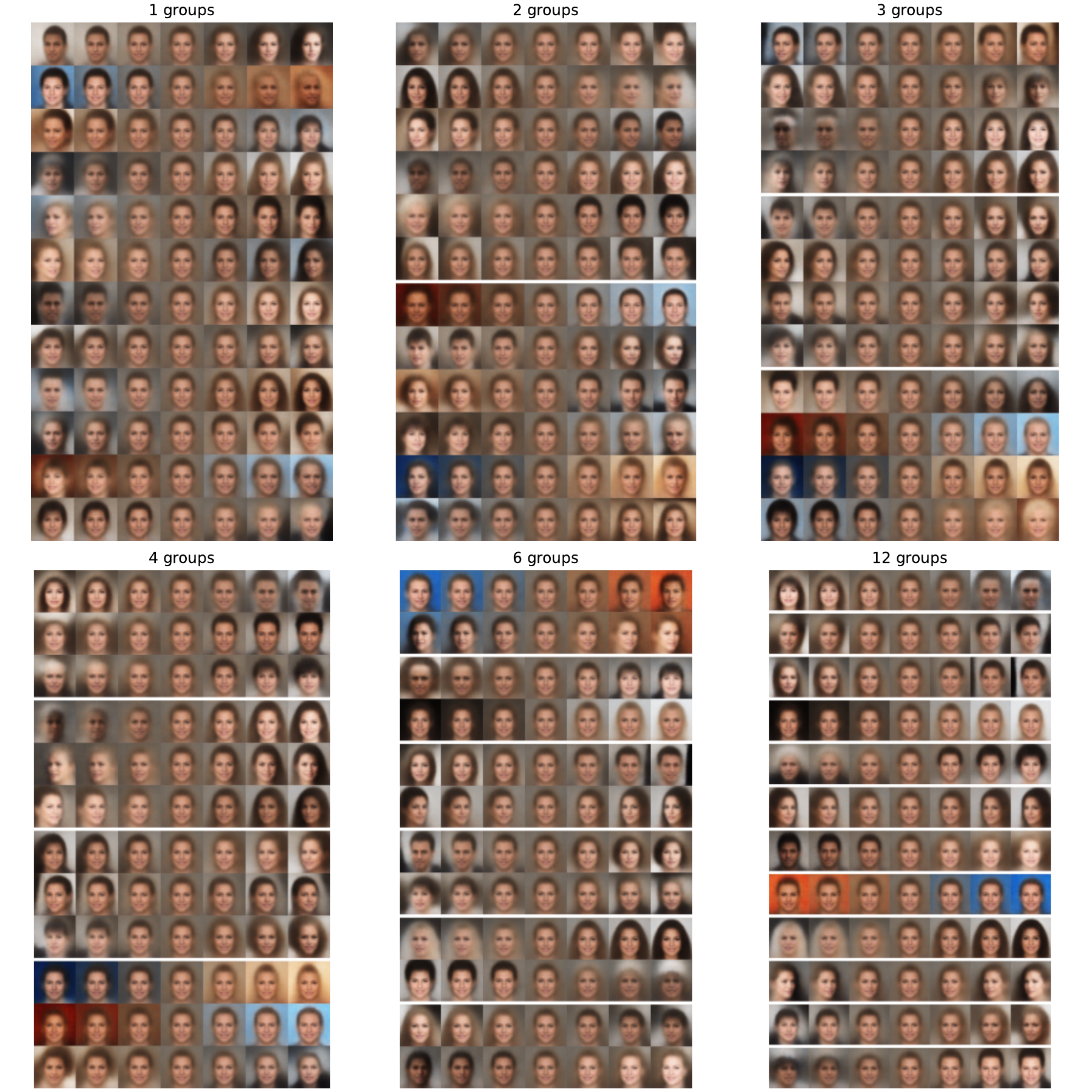}
    \vspace{-0.3in}
    \caption{The reconstructed images by varying one of the $K=12$ disentangled latent from applying PDisVAE to the CelebA dataset with the different number of groups $G\in\{1,2,3,4,6,12\}$. When $G=1$, PDisVAE becomes the standard VAE; when $G=K=12$, PDisVAE becomes the fully entangled VAE (e.g., $\beta$-TCVAE or FactorVAE). In each plot, each row is by varying one latent component (latent dimension) while fixing all others to 0s.}
    \label{fig:celeba_all}
\end{figure}

\clearpage


\end{document}